\documentclass{article}

\usepackage[main,final]{neurips_2026}

\usepackage[utf8]{inputenc}
\usepackage[T1]{fontenc}

\usepackage{microtype}
\usepackage{graphicx}
\usepackage{booktabs}
\usepackage{amsmath}
\usepackage{amssymb}
\usepackage{mathtools}
\usepackage{amsfonts}
\usepackage{amsthm}
\usepackage{xcolor}
\usepackage{url}
\usepackage[hidelinks]{hyperref}

\usepackage{algorithm}
\usepackage{algpseudocode}

\usepackage{threeparttable}
\usepackage{adjustbox}
\usepackage{tabularx}
\usepackage{multirow}
\usepackage{makecell}
\usepackage{rotating}
\usepackage{tikz}
\usepackage{textcomp}
\usepackage{bm}
\usepackage{bbm}
\usepackage{mathrsfs}
\usepackage{wrapfig}
\usepackage{soul}
\usepackage{placeins}
\usepackage{float}
\usepackage[textsize=tiny]{todonotes}

\newtheorem{theorem}{Theorem}[section]
\newtheorem{corollary}[theorem]{Corollary}
\newtheorem{lemma}[theorem]{Lemma}

\newtheorem{assumption}{Assumption}[section]
\newcommand{\loss}{\mathrm{loss}}

\title{SCDM: Scalable Causal Discovery in Nonlinear Temporal Systems with Meta-Learning}

\author{%
  Jingbo Wang \quad Kegeng Tang \quad Zihao Wang \quad
  Shaogang Ren\thanks{Corresponding author:  \texttt{sren9@tennessee.edu}.} \\
  University of Tennessee at Chattanooga \\
  \texttt{\{ydk297,hcw575,zihao.wang,sren9\}@tennessee.edu}
}

\begin{document}



\maketitle

\begin{abstract}
Causal discovery from nonlinear multivariate time series is challenging in high-dimensional systems, where the number of candidate directed relations grows quadratically with the number of variables. Existing methods are often limited by target-wise model fitting, repeated conditional testing, or expensive graph-search procedures. We propose SCDM, a shared-parameter meta-learning framework for scalable temporal causal discovery. SCDM treats each target variable as a task while learning a shared temporal predictor and a shared causal-strength matrix across tasks. Each meta-episode updates only a subset of target-variable tasks, allowing the model to accumulate structural evidence across the full system without exhaustive target-wise optimization. After training, a post-training graph readout converts the learned predictor into continuous directed causal scores for threshold-independent evaluation. We also provide a finite-error recovery principle showing that, under fully observed delayed-system assumptions and a positive risk-gap condition, thresholding SCDM scores recovers the graph when statistical, meta-optimization, and readout errors are below the population margin. Experiments on controlled synthetic, neuroimaging-inspired, high-dimensional, and real-world temporal benchmarks demonstrate competitive causal recovery and improved scalability in large temporal systems.
\end{abstract}

\begin{center}
  \small
  \textbf{Project page:}\quad
  \url{https://shaogangren.github.io/SCDM/}\\[2pt]
  \textbf{Code:}\quad
  \url{https://github.com/ShaogangRen/SCDM}
\end{center}

\section{Introduction}
\label{sec:introduction}

Causal discovery from multivariate time series seeks directed dependencies in systems such as brain activity, financial markets, and traffic networks. Temporal ordering supports predictive causal analysis through the Granger framework~\citep{granger1969investigating,lutkepohl2005new}. For $d$ observed variables, however, the candidate cross-variable edge space contains $d(d-1)$ ordered pairs, so nonlinear prediction, target-wise fitting, and repeated conditional tests become costly in high-dimensional systems~\citep{runge2019detecting,runge2020discovering}.

Existing methods address parts of this challenge through structured sparsity, temporal convolutions, self-explaining predictors, perturbation-based scores, or coarse-to-fine discovery~\citep{tank2021neural,nauta2019causal,marcinkevivcs2021interpretable,bi2025uncle,cheng2024cuts+}. Many of them still require target-specific optimization or repeated graph evaluation, which becomes limiting when the edge space contains millions of directed pairs. The question is therefore how to share nonlinear prediction across targets while keeping target-specific structural evidence.

We propose Scalable Causal Discovery with Meta-Learning (SCDM), a shared-parameter framework for nonlinear temporal causal discovery. Inspired by gradient-based meta-learning~\citep{finn2017model,nichol2018first,yang2022calibrating}, SCDM treats the prediction of each target variable as a related task in the same dynamical system. It learns a shared nonlinear temporal predictor and a shared causal-strength matrix. Each meta-episode samples a small batch of target-variable tasks, performs temporary support adaptation, and updates the shared parameters through query losses. Repeated over episodes, this procedure accumulates structural evidence across variables while keeping each update tied to the sampled task batch rather than the full set of targets. This target-task subsampled meta-training is the main scalability mechanism of SCDM.

SCDM separates training from graph extraction. After meta-training, a post-training perturbation-style readout, related to recent perturbation-based dynamic causal scoring~\citep{bi2025uncle}, converts the trained predictor into continuous directed source--target scores. This readout is a graph-extraction interface rather than the main training contribution. We also provide a finite-error recovery principle for the fully observed delayed temporal setting: if statistical, meta-optimization, and readout errors are smaller than the population risk-gap margin, thresholding the learned scores separates causal parents from non-parents.

We evaluate SCDM on controlled nonlinear benchmarks with known ground-truth graphs, The Virtual Brain (TVB)-inspired neuroimaging simulations~\citep{sanz2013virtual}, a high-dimensional voxel-level functional magnetic resonance imaging (fMRI) benchmark, and the real-world LargeST traffic benchmark~\citep{liu2023largest}. On controlled benchmarks, SCDM is competitive with strong baselines rather than uniformly best. Its main advantage appears at large scale: it gives a better runtime--accuracy tradeoff on 1{,}000 variables and finishes on 3{,}000 variables with far less time and memory. Two ablations separate the roles of the two training components: target-task subsampling lowers training cost without hurting recovery, and the support/query meta-objective yields shared parameters that adapt better to held-out targets than pooled multi-task training.

The main contributions are:
\begin{itemize}
    \item We formulate nonlinear temporal causal discovery as a shared-parameter episodic meta-learning problem. Each target variable defines one prediction task, and all tasks share one temporal predictor and one causal-strength matrix, so no separate model is fitted for each target.
    \item We separate two training components with different roles. Target-task subsampling sets how many target tasks each update touches ($B$ instead of $d$), which controls training cost. The support/query objective trains the shared parameters to perform well after a few target-specific steps, which pooled multi-task training does not do. We test each component with its own ablation.
    \item We give a finite-error recovery result for fully observed delayed temporal systems, and evaluate SCDM on controlled, neuroimaging-inspired, voxel-level ($d=1000$), and real-world ($d=3000$) benchmarks, together with stress tests for collinear sources and withheld variables.
\end{itemize}

\section{Related Work}
\label{sec:related_work}

\paragraph{Temporal causal discovery.}
Temporal causal discovery exploits lagged dependencies to infer directed relations in multivariate time series. Classical Granger causality defines directed influence by whether one variable's history improves another variable's prediction~\citep{granger1969investigating}, with VAR models providing a standard linear implementation~\citep{lutkepohl2005new}. VARLiNGAM further combines vector autoregression with non-Gaussianity assumptions for structural identification~\citep{hyvarinen2010estimation}. Constraint-based methods such as PCMCI and PCMCI+ use conditional independence and momentary conditional independence tests for autocorrelated time series~\citep{runge2019detecting,runge2020discovering}, with related work studying time-series-specific assumptions and conditional independence structures~\citep{entner2010causal,peters2013causal,malinsky2018causal}. These methods provide important foundations, but their linear, component-wise, or repeated-testing structures can become costly in nonlinear high-dimensional systems where the candidate edge space scales as $d(d-1)$. SCDM keeps the predictive interpretation of temporal causality while replacing repeated target-wise fitting or testing with a shared meta-learned temporal predictor and causal-strength matrix.

\paragraph{Neural temporal causal discovery.}
Neural methods improve nonlinear modeling capacity by replacing linear predictors with expressive sequence models. Neural Granger causality uses component-wise predictors with structured sparsity~\citep{tank2021neural}; GVAR uses self-explaining neural networks for nonlinear and time-varying Granger effects~\citep{marcinkevivcs2021interpretable}; and TCDF combines temporal convolutions, attention-based interpretation, and causal validation~\citep{nauta2019causal}. More general neural causal discovery methods use differentiable graph learning, continuous optimization, and score-based discovery for nonlinear systems~\citep{zheng2018dags,kalainathan2018sam,ng2022masked,zhang2024learning,zhang2025invargc}. Recent perturbation-based methods such as UnCLe estimate dynamic causal effects through uncoupled temporal modeling and perturbation-based causal scoring~\citep{bi2025uncle}, related to occlusion-style attribution under input removal or perturbation~\citep{zeiler2014visualizing,ribeiro2016should,shrikumar2017learning}. SCDM is related to these methods through nonlinear prediction and continuous graph scoring, but uses perturbation-style scoring only as a post-training readout; its main scalability mechanism is shared episodic meta-training across target-variable tasks.

\paragraph{Causal inference scalability and meta-learning.}
Existing scalable strategies include sparsity regularization and fast sparse solvers~\citep{ren2020thunder,ren2024dynamic}, variable-wise decomposition, screening~\citep{ren2017safe,ren2017scalable}, and coarse-to-fine search; for example, neural Granger methods use structured sparsity~\citep{tank2021neural}, while CUTS+ targets high-dimensional irregular time series through coarse-to-fine discovery and message-passing prediction modules~\citep{cheng2024cuts+}. SCDM instead formulates target-variable prediction as a family of related tasks. Inspired by MAML and first-order meta-learning~\citep{finn2017model,nichol2018first}, and related meta-learning applications in few-shot learning, optimization, and structured prediction~\citep{vinyals2016matching,ravi2017optimization,yao2019hierarchically}, SCDM samples a subset of target-variable tasks per episode and updates a shared temporal predictor together with a shared causal-strength matrix through query losses. This support/query formulation is used not for few-shot classification, but as a scalable mechanism for accumulating structural evidence across variables without updating all target tasks at every iteration.

\section{Problem Formulation}
\label{sec:problem}

\paragraph{Problem setup and target-variable tasks.}
\label{subsec:problem_setup}
Let $X=\{x_t\}_{t=1}^{T}$ be a $d$-dimensional time series, with $x_t\in\mathbb{R}^{d}$. For observed-history order $p$, define $u_t=[x_{t-1},\ldots,x_{t-p}]\in\mathbb{R}^{dp}$ for $t=p+1,\ldots,T$. Task $i$ predicts the target $y_t^{(i)}=x_t^{(i)}$ from this history. Following support/query meta-learning~\citep{finn2017model}, the $d$ target variables define $d$ related tasks within the same dynamical process.

All tasks share the parameters $\Theta=\{\theta,W\}$, where $\theta$ parameterizes a nonlinear temporal predictor and $W\in\mathbb{R}^{d\times d}$ is a causal-strength matrix with $W_{j,i}$ the strength from source $j$ to target $i$. Each task starts from this shared initialization and adapts only a temporary local copy; no private target-specific head, embedding, or persistent adapted model is introduced. Causal structure is therefore represented once, globally, by $W$, and scalability comes from updating only sampled target tasks.

Figure~\ref{fig:method_overview} summarizes the shared-parameter episodic pipeline, and Table~\ref{tab:scdm_notation} at the start of the appendix lists all notation.

\begin{figure}[t]
    \centering
    \includegraphics[width=0.98\textwidth]{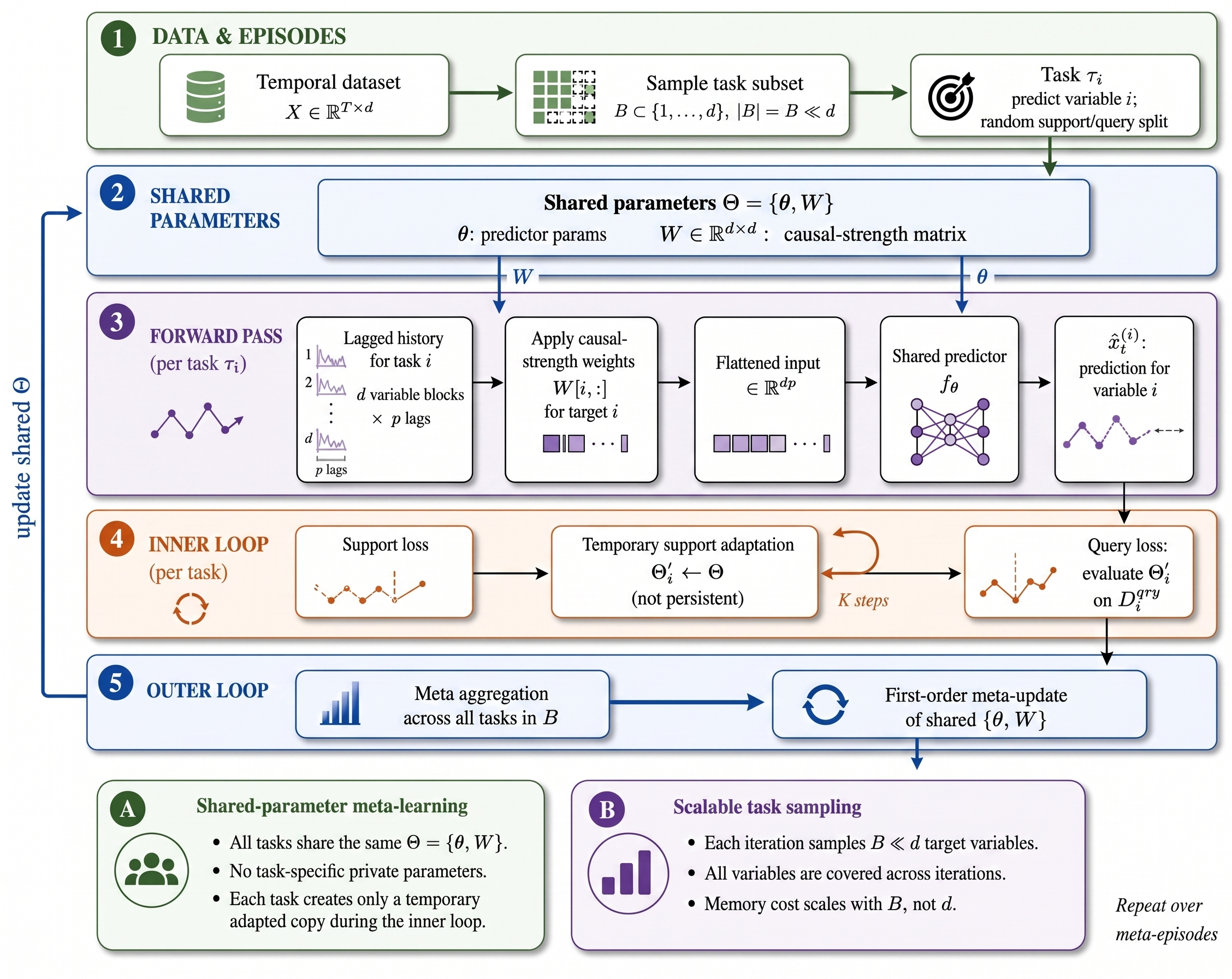}
    \caption{Overview of SCDM. Each target variable is treated as a task; sampled tasks undergo temporary support adaptation, and their query losses update the shared temporal predictor and causal-strength matrix.}
    \label{fig:method_overview}
\end{figure}

\paragraph{Structured predictor over observed variables.}
\label{subsec:nonlatent_predictor}
For the fully observed temporal system considered in the main method, the lagged history of each source variable \(j\) is modulated by a shared variable-level structural weight \(W_{j,i}\) when predicting target variable \(i\). Specifically, the weighted lagged history \(\widetilde{H}^{(i)}_t\in\mathbb{R}^{d\times p}\) has entries \(\widetilde{H}^{(i)}_{t,j,\ell}=W_{j,i}x^{(j)}_{t-\ell}\) for \(j=1,\ldots,d\) and \(\ell=1,\ldots,p\). The weighted history is flattened and passed to the shared predictor as \(\widehat{x}^{(i)}_t=f_\theta(\mathrm{vec}(\widetilde{H}^{(i)}_t))\). Since \(W_{j,i}\) gates all lagged coordinates of source \(j\) when predicting target \(i\), the learned matrix \(W\) provides a shared variable-level directed causal-strength representation, which is later converted into continuous graph scores by the post-training readout.

\section{Methodology}
\label{sec:method}

\subsection{Meta-Learning for Causal Discovery in Nonlinear Temporal Systems}
\label{sec:meta-method}

SCDM alternates temporary target-task adaptation with persistent shared updates. Each episode samples a target-task batch $\mathcal{B}$ and support/query windows for its tasks. Starting from the shared initialization $\Theta$, each task performs support updates and evaluates the adapted parameters on its query window. The averaged query loss then updates the persistent shared parameters, and the adapted task copies are discarded. The objectives below specify this procedure.

For each target variable \(i\), let \(\mathcal{D}^{\mathrm{sup}}_i\) and
\(\mathcal{D}^{\mathrm{qry}}_i\) denote support and query temporal windows, and let \(\Theta=\{\theta,W\}\) be the shared parameters. The support objective is
\begin{equation}
\mathcal{L}^{\mathrm{sup}}_i(\Theta)
=
\frac{1}{|\mathcal{D}^{\mathrm{sup}}_i|}
\sum_{(u_t,y^{(i)}_t)\in\mathcal{D}^{\mathrm{sup}}_i}
\loss\!\left(
f_\theta\!\left(\mathrm{vec}\!\left(\widetilde{H}^{(i)}_t\right)\right),
y^{(i)}_t
\right)
+\lambda_W\|W[:,i]\|_1 .
\label{eq:support_loss}
\end{equation}
Here \(\widetilde{H}^{(i)}_t\) is the weighted lagged history for target \(i\), constructed using the source-to-target weights \(W_{j,i}\). Starting from the shared initialization, task \(i\) performs \(K\) temporary support updates:
\begin{equation}
\Theta^{(0)}_i \leftarrow \Theta,
\qquad
\Theta^{(k+1)}_i
=
\Theta^{(k)}_i
-
\alpha\nabla_{\Theta^{(k)}_i}
\mathcal{L}^{\mathrm{sup}}_i(\Theta^{(k)}_i),
\quad k=0,\ldots,K-1 .
\label{eq:inner_update}
\end{equation}
The adapted parameters \(\Theta'_i=\Theta^{(K)}_i\) are used only within the current episode; they are not stored as target-specific models. The query loss is
\begin{equation}
\mathcal{L}^{\mathrm{qry}}_i(\Theta'_i)
=
\frac{1}{|\mathcal{D}^{\mathrm{qry}}_i|}
\sum_{(u_t,y^{(i)}_t)\in\mathcal{D}^{\mathrm{qry}}_i}
\loss\!\left(
f_{\theta'_i}\!\left(
\mathrm{vec}\!\left(\widetilde{H}^{(i)}_t(\Theta'_i)\right)
\right),
y^{(i)}_t
\right)
+\lambda_W\|W'_i[:,i]\|_1 .
\label{eq:query_loss}
\end{equation}
The $\ell_1$ terms in \eqref{eq:support_loss} and \eqref{eq:query_loss} push the weights of uninformative sources toward zero~\citep{tibshirani1996regression}, which matches the expectation that each target has only a few true parents. For a sampled target-task mini-batch \(\mathcal{B}\), we define the meta-objective as:
\begin{equation}
\mathcal{L}_{\mathrm{meta}}
=
\frac{1}{|\mathcal{B}|}
\sum_{i\in\mathcal{B}}
\mathcal{L}^{\mathrm{qry}}_i(\Theta'_i).
\label{eq:meta_objective}
\end{equation}
Only the query objective updates the persistent shared parameters; the task-local adapted parameters are discarded after the episode.

We use a first-order MAML approximation~\citep{finn2017model,nichol2018first},
\begin{equation}
\nabla_{\Theta}\mathcal{L}^{\mathrm{qry}}_i(\Theta'_i)
\approx
\nabla_{\Theta'_i}\mathcal{L}^{\mathrm{qry}}_i(\Theta'_i),
\qquad
\Theta \leftarrow \Theta-\beta\nabla_{\Theta}\mathcal{L}_{\mathrm{meta}} .
\label{eq:outer_update}
\end{equation}
This first-order update avoids differentiation through the inner optimization trajectory; its cost still grows with the predictor input dimension $dp$ and with the number of sampled targets $B$. The two parts of this procedure play different roles. Target-task subsampling sets how many tasks each update touches, which controls training cost. The support/query split makes each query loss measure how well the shared parameters perform after $K$ target-specific steps, so the outer update favors parameters that adapt well to every target rather than parameters that only fit the average target. Section~\ref{sec:ablation} tests the two roles separately. Both the support and query losses use the same sparsity coefficient $\lambda_W$.

\begin{algorithm}[t]
\caption{SCDM meta-training for scalable temporal causal discovery}
\label{alg:meta_main}
\begin{algorithmic}[1]
\Require Time series \(X_{1:T}\), lag order \(p\), task batch size \(B\), inner steps \(K\)
\State Initialize shared parameters \(\Theta=\{\theta,W\}\)
\Repeat
    \State Sample target-task batch \(\mathcal{B}\subset\{1,\ldots,d\}\)
    \For{\(i\in\mathcal{B}\)}
        \State Sample support/query windows \(\mathcal{D}^{\mathrm{sup}}_i,\mathcal{D}^{\mathrm{qry}}_i\)
        \State Adapt temporarily: \(\Theta'_i \leftarrow \mathrm{Adapt}_K(\Theta,\mathcal{D}^{\mathrm{sup}}_i)\)
        \State Compute query loss \(\mathcal{L}^{\mathrm{qry}}_i(\Theta'_i)\)
    \EndFor
    \State Meta-update shared \(\Theta\) using averaged query loss over \(\mathcal{B}\)
\Until{convergence}
\State \Return learned shared predictor \(f_\theta\) and causal-strength matrix \(W\)
\end{algorithmic}
\end{algorithm}

Algorithm~\ref{alg:meta_main} summarizes the training procedure. Task sampling reduces the number of target losses processed in each episode, while the cost of each predictor evaluation and the storage of $W$ still grow with $d$; Appendix~\ref{app:baseline_hparams} separates training, storage, and readout costs.

\subsection{Post-training Graph Score Readout}
\label{sec:score_readout}

After meta-training, SCDM converts the learned shared predictor and causal-strength parameterization into directed source--target scores through a post-training graph readout. This readout is not part of the optimization objective; it is the graph-extraction interface for the trained meta-predictor. We use two readouts. The source-isolated readout below measures how much each candidate source alone helps the trained predictor. The structural readout uses $|W_{j,i}|$ directly; it needs no extra model evaluation, which makes it the only affordable choice when $d\ge 1000$. Each main-text recovery table states its readout, and Appendix~\ref{app:supplementary_results} compares both on the same trained models.

The source-isolated readout is defined as follows. For a candidate source variable \(j\) and target variable \(i\), we retain only the lagged history of source \(j\) and set all other source histories to zero. Let
\(L^{(i)}_{\mathrm{zero}}\) denote the prediction loss under a zero-input baseline, and let \(L^{(i)}_j\) denote the loss when only source \(j\) is retained. The directed score is
\begin{equation}
s_{j\to i}
=
\frac{
L^{(i)}_{\mathrm{zero}} - L^{(i)}_j
}{
L^{(i)}_{\mathrm{zero}}+\epsilon
},
\label{eq:readout_score}
\end{equation}
where \(\epsilon>0\) is a small numerical constant. Larger values indicate that source \(j\) yields a larger reduction in prediction loss for target \(i\), and are therefore interpreted as stronger directed temporal influence. This readout is related to perturbation-based temporal causal scoring~\citep{bi2025uncle}, but the main scalability mechanism of SCDM is the shared episodic meta-training procedure. Full implementation details are provided in Appendix~\ref{app:score_readout_details}. Both readouts measure predictive contribution rather than the effect of an intervention. When two sources carry nearly the same information, the predictor can rely on either one, so the score of an individual edge may move to a correlated non-parent; Appendix~\ref{app:collinear} measures this effect with a controlled collinearity test.

\subsection{Finite-Error Graph Recovery}
\label{sec:theory}

We analyze the fully observed non-latent setting to characterize when the scores produced by SCDM recover the delayed causal graph. The analysis follows the standard risk-gap view of temporal causal discovery and is stated under delayed causal sufficiency, correct lag specification, and sufficient regularity of the time-series process. Under these conditions, thresholding the learned scores recovers the graph whenever the statistical, meta-optimization, and readout
errors are smaller than the population separation margin.

For each target variable \(i\), SCDM treats one-step prediction of \(X_t^{(i)}\) as a task. Let \(H_t\in\mathbb{R}^{d\times p}\) be the full lagged history and \(H_{t,-j}\) the history with all lags of source \(j\) removed. Define the population full risk \(R_i^\star=\inf_{f\in\mathcal{F}_i}\mathbb{E}[\loss(f(H_t),X_t^{(i)})]\) and the source-restricted risk \(R_{i,-j}^\star=\inf_{f\in\mathcal{F}_{i,-j}}\mathbb{E}[\loss(f(H_{t,-j}),X_t^{(i)})]\), where \(\mathcal{F}_{i,-j}\) excludes predictors using source \(j\)'s lagged history. The population score is \(S^\star_{ji}=R^\star_{i,-j}-R^\star_i\), measuring the optimal risk increase for target task \(i\) when source \(j\) is removed. We assume a positive delayed-system margin
\begin{equation}
\Delta_{\min}=\min_{(j,i):\,j\in\mathrm{Pa}(i)}S^\star_{ji}>0,\qquad S^\star_{ji}=0\ \text{ for } j\notin\mathrm{Pa}(i).
\label{eq:min_gap_main}
\end{equation}

In SCDM, the practical score is obtained after a finite meta-training procedure and a post-training readout. Let \(\widetilde{S}_{ji,T,N}\) denote this score after observing \(T\) time points and running \(N\) outer meta-updates with task
batch size \(B\) and \(K\) inner steps. Its deviation from the population risk-gap score can be decomposed as
\begin{equation}
\max_{i,j}
\left|
\widetilde{S}_{ji,T,N}
-
S^\star_{ji}
\right|
\le
\epsilon_{\mathrm{stat}}(T,d,\delta)
+
\epsilon_{\mathrm{meta}}(N,B,K)
+
\epsilon_{\mathrm{readout}}(T,N).
\label{eq:main_error_decomposition}
\end{equation}
Here \(\epsilon_{\mathrm{stat}}\) is the finite-sample error under temporal dependence; \(\epsilon_{\mathrm{meta}}\) is the SCDM-specific residual from finite episodic meta-training, task subsampling, support adaptation, first-order meta-updates, and shared query-loss optimization; and \(\epsilon_{\mathrm{readout}}\) is the gap between the implemented readout and the ideal empirical source-ablation score.

\begin{theorem}[Finite-error recovery of SCDM scores]
\label{thm:finite_error_graph_recovery}
Consider a fully observed delayed temporal system satisfying causal sufficiency, correct lag specification, stationarity with sufficient mixing, and the population margin condition above. Suppose that, with probability at least \(1-\delta\), the SCDM score satisfies Eq.~\eqref{eq:main_error_decomposition}. If
\begin{equation}
\epsilon_{\mathrm{stat}}(T,d,\delta)
+
\epsilon_{\mathrm{meta}}(N,B,K)
+
\epsilon_{\mathrm{readout}}(T,N)
<
\frac{\Delta_{\min}}{2},
\label{eq:recovery_condition_main}
\end{equation}
then thresholding \(\widetilde{S}_{ji,T,N}\) at any \(\gamma\in(\epsilon_{\mathrm{tot}},\Delta_{\min}-\epsilon_{\mathrm{tot}})\), where \(\epsilon_{\mathrm{tot}}=\epsilon_{\mathrm{stat}}(T,d,\delta)+\epsilon_{\mathrm{meta}}(N,B,K)+\epsilon_{\mathrm{readout}}(T,N)\), recovers the true delayed graph with probability at least \(1-\delta\). In particular, \(\gamma=\Delta_{\min}/2\) is valid.
\end{theorem}

Theorem~\ref{thm:finite_error_graph_recovery} separates identifiability from finite-error estimation. The margin $\Delta_{\min}$ belongs to the delayed data-generating process, while SCDM enters through the meta-optimization and readout residuals. Task-subsampled meta-training is the estimation procedure, and correct recovery still requires these residuals to be small. The theorem gives no comparative guarantee between meta-training and other training objectives; Section~\ref{sec:ablation} compares them empirically. The proof and statistical-rate instantiation are given in Appendix~\ref{app:theory}.

Causal sufficiency is the assumption most likely to fail in practice. If an unobserved variable drives both $j$ and $i$, the history of $j$ carries information about that hidden driver, so removing $j$ raises the risk of target $i$ even when $j\notin\mathrm{Pa}(i)$. Then $S^\star_{ji}>0$ for a non-edge, the margin condition fails, and the theorem gives no guarantee for that pair. When the structural readout $|W_{j,i}|$ is used, $\epsilon_{\mathrm{readout}}$ is the gap between $|W_{j,i}|$ and the ablation score, which we do not bound; for that readout the theorem states a condition rather than a guarantee.

\section{Experiments}
\label{sec:experiments}

\subsection{Experimental Setup}
\label{sec:exp_setup}

\paragraph{Datasets.} Our benchmarks serve three evaluation roles. NC8, simulated EEG, and ROI-level simulated fMRI provide known graphs for controlled recovery evaluation. The voxel-level fMRI simulation provides a larger known graph for runtime and recovery studies; it also contains latent regional drivers. LargeST provides real traffic data for computational feasibility, forecasting, and spatial-plausibility diagnostics, without complete directed causal ground truth. Dataset statistics and split protocols appear in Appendix~\ref{app:dataset_overview}.

\paragraph{Baselines.} We compare six temporal discovery methods. UnCLe uses uncoupled temporal modeling and perturbation-based scoring~\citep{bi2025uncle}; PCMCI uses conditional-independence testing~\citep{runge2019detecting}; TCDF combines temporal convolution, attention, and validation~\citep{nauta2019causal}. GVAR uses self-explaining autoregressive coefficients~\citep{marcinkevivcs2021interpretable}, cMLP uses structured sparsity in component-wise predictors~\citep{tank2021neural}, and CUTS+ uses coarse-to-fine graph discovery for high-dimensional series~\citep{cheng2024cuts+}. Methods define temporal context differently (an observed-history order, a latent autoregressive order, or a convolution kernel); Table~\ref{tab:baseline_temporal} lists each setting and how it was chosen, and Appendix~\ref{app:baseline_hparams} gives all configurations. The 1{,}000-variable study compares against UnCLe and CUTS+, which output continuous scores and save intermediate checkpoints and can therefore be evaluated at matched wall-clock budgets (Appendix~\ref{app:voxel_fmri_results}).

\paragraph{Evaluation metrics.}
For datasets with known ground-truth graphs, we evaluate causal discovery using AUROC and AUPRC as threshold-independent metrics.
We additionally report F1 and SHD for thresholded graphs in the appendix.
Unless otherwise stated, self-edges are excluded from graph-level evaluation because they correspond to autoregressive self-dependencies rather than cross-variable causal discovery.
For real-world LargeST experiments, where a definitive causal graph is unavailable, we report runtime, peak memory, forecasting error, and qualitative learned-graph visualizations rather than causal AUROC or AUPRC.

\paragraph{Runtime and implementation protocol.}
Runtime is measured as wall-clock time under a fixed hardware and software environment.
Unless otherwise stated, runtime includes model training and causal score extraction, but excludes dataset preprocessing and figure generation.
For high-dimensional experiments, we additionally record peak GPU memory when available.
A method is marked as OOM if it fails under the specified memory budget, and as timeout if it does not complete within the fixed wall-clock limit.
Detailed hyperparameters, compute resources, thresholding rules, and runtime measurement details are provided in Appendix~\ref{app:baseline_hparams} and Appendix~\ref{app:additional_runtime_metrics}.

\subsection{Causal Recovery on Controlled Benchmarks}
\label{sec:controlled_recovery}

\begin{table}[t]
\centering
\caption{
Causal discovery performance on low- and medium-dimensional temporal benchmarks with known ground-truth graphs.
We report AUROC and AUPRC as threshold-independent metrics.
Results are shown as mean $\pm$ 95\% confidence interval.
Best results are shown in \textbf{bold} and second-best results are \underline{underlined}.
}
\label{tab:low_medium_causal_recovery}
\resizebox{\textwidth}{!}{
\begin{tabular}{lcccccc}
\toprule
\multirow{2}{*}{Method}
& \multicolumn{2}{c}{NC8}
& \multicolumn{2}{c}{Simulated EEG}
& \multicolumn{2}{c}{Simulated fMRI} \\
\cmidrule(lr){2-3}
\cmidrule(lr){4-5}
\cmidrule(lr){6-7}
& AUROC $\uparrow$ & AUPRC $\uparrow$
& AUROC $\uparrow$ & AUPRC $\uparrow$
& AUROC $\uparrow$ & AUPRC $\uparrow$ \\
\midrule

UnCLe
& \textbf{0.947 $\pm$ 0.014}
& \textbf{0.756 $\pm$ 0.080}
& 0.571 $\pm$ 0.008
& 0.155 $\pm$ 0.008
& \underline{0.661 $\pm$ 0.012}
& \underline{0.268 $\pm$ 0.006} \\

PCMCI
& 0.909 $\pm$ 0.003
& 0.594 $\pm$ 0.022
& \underline{0.743 $\pm$ 0.012}
& \underline{0.390 $\pm$ 0.018}
& 0.569 $\pm$ 0.012
& 0.160 $\pm$ 0.010 \\

TCDF
& 0.740 $\pm$ 0.019
& 0.415 $\pm$ 0.020
& 0.685 $\pm$ 0.005
& 0.224 $\pm$ 0.007
& 0.585 $\pm$ 0.019
& 0.178 $\pm$ 0.020 \\

GVAR
& 0.750 $\pm$ 0.013
& 0.413 $\pm$ 0.041
& \textbf{0.750 $\pm$ 0.004}
& 0.344 $\pm$ 0.019
& 0.619 $\pm$ 0.017
& 0.219 $\pm$ 0.019 \\

CUTS+
& 0.795 $\pm$ 0.025
& 0.354 $\pm$ 0.035
& 0.589 $\pm$ 0.013
& 0.160 $\pm$ 0.007
& 0.532 $\pm$ 0.009
& 0.141 $\pm$ 0.009 \\

cMLP
& 0.753 $\pm$ 0.002
& 0.524 $\pm$ 0.006
& 0.656 $\pm$ 0.014
& 0.258 $\pm$ 0.024
& 0.624 $\pm$ 0.013
& 0.240 $\pm$ 0.020 \\

Ours
& \underline{0.916 $\pm$ 0.015}
& \underline{0.704 $\pm$ 0.035}
& 0.717 $\pm$ 0.037
& \textbf{0.476 $\pm$ 0.067}
& \textbf{0.867 $\pm$ 0.006}
& \textbf{0.728 $\pm$ 0.019} \\

\bottomrule
\end{tabular}
}
\end{table}

Table~\ref{tab:low_medium_causal_recovery} shows competitive but not uniformly superior recovery. On NC8, SCDM ranks second to UnCLe in both AUROC and AUPRC. On simulated EEG, it achieves the highest AUPRC, whereas GVAR and PCMCI achieve higher AUROC. On simulated fMRI, SCDM achieves the highest AUROC and AUPRC among the compared methods. The comparison shows that sharing one predictor and one causal-strength matrix across targets does not cost recovery accuracy at moderate scale; the gains of SCDM appear at larger scale (Sections~\ref{sec:voxel_fmri_scalability} and~\ref{sec:largest_scalability}).

\begin{figure*}[t]
\centering
\includegraphics[width=0.98\textwidth]{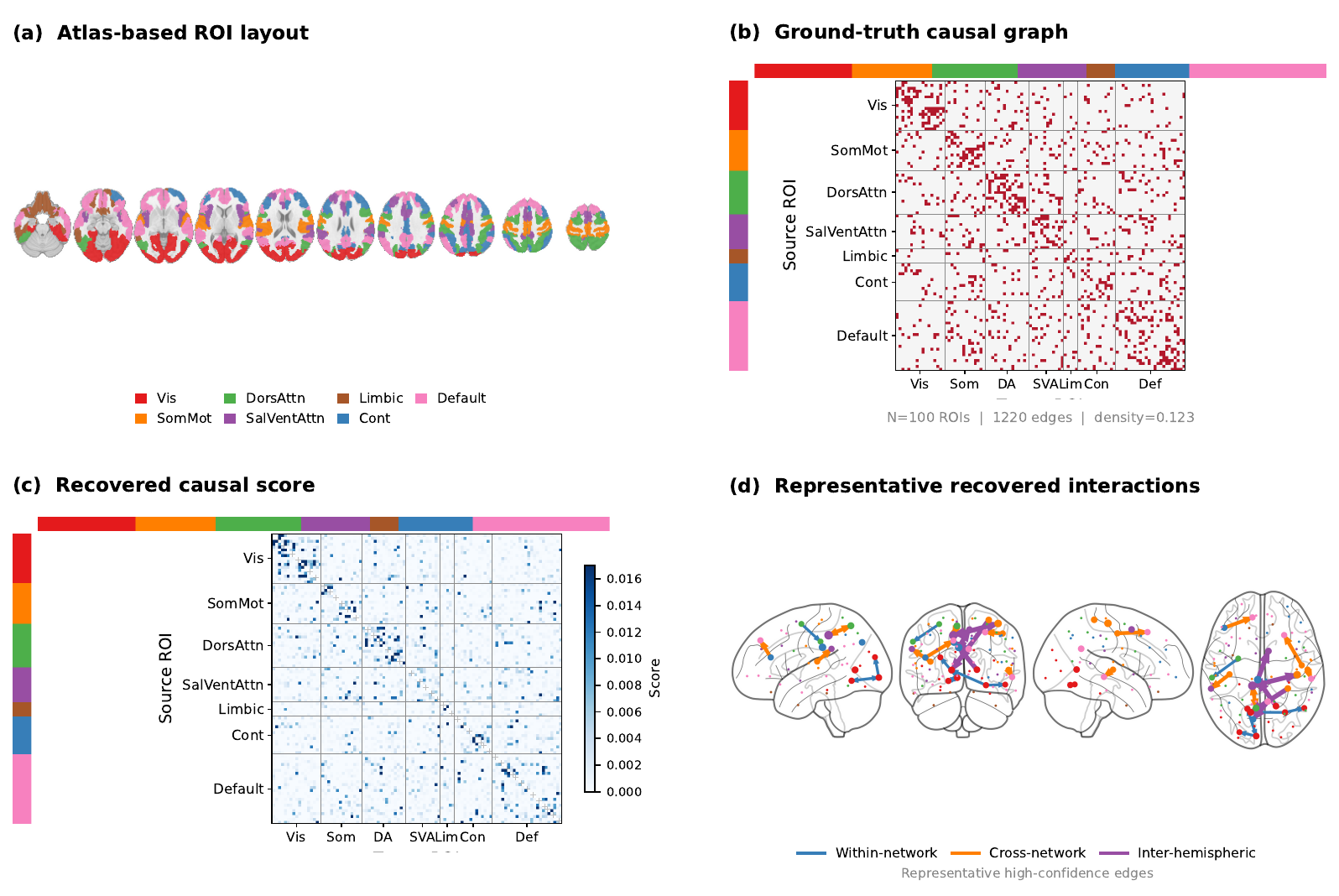}
\caption{Qualitative causal recovery on the atlas-grounded simulated fMRI benchmark. 
(a) Schaefer-100 ROI layout with Yeo-network labels. 
(b) Network-sorted ground-truth directed graph. 
(c) Continuous causal scores recovered by SCDM, shown in the same source--target ordering as (b). 
(d) Representative high-confidence recovered interactions mapped back to ROI space.}
\label{fig:fmri_qualitative}
\end{figure*}

Figure~\ref{fig:fmri_qualitative} shows the ground truth and the recovered scores in the same network-sorted source--target order, so one can check whether the recovered scores keep the network-organized structure of the simulation; panel (d) maps representative recovered interactions back to ROI space.

\subsection{High-Dimensional Scalability on Synthetic Voxel-level fMRI}
\label{sec:voxel_fmri_scalability}

To test scalability beyond ROI-level benchmarks, we evaluate SCDM on a synthetic voxel-level fMRI benchmark with $d=1000$ observed variables, $999{,}000$ off-diagonal candidate directed relations, and $7{,}142$ true edges (density $0.71\%$). Its known observed graph allows recovery evaluation at this scale. The simulator also includes latent regional drivers, so this experiment is an empirical scalability test rather than a direct test of the fully observed recovery theorem. Because a random ranking has an AUPRC equal to the edge density, we report AUPRC as enrichment over random. We compare against UnCLe and CUTS+, which output continuous scores and save intermediate checkpoints, at matched wall-clock budgets.

\begin{table}[t]
\centering
\small
\caption{Fixed-budget results on voxel-level fMRI ($d=1000$; mean over three seeds). Each entry is AUROC / AUPRC enrichment over random (random AUPRC $=0.0071$). At each budget we evaluate the latest checkpoint finished within that budget, so equal values at consecutive budgets mean that no new checkpoint finished in between. SCDM uses the structural readout $|W_{j,i}|$. Best value per budget in \textbf{bold}.}
\label{tab:voxel_fmri_fixed_budget}
\setlength{\tabcolsep}{5pt}
\begin{tabular}{@{}lcccc@{}}
\toprule
Method & 200\,s & 400\,s & 600\,s & 800\,s \\
\midrule
UnCLe (lag 1) & \textbf{0.573} / \textbf{2.05}$\times$ & 0.556 / 1.85$\times$ & 0.556 / 1.85$\times$ & 0.541 / 1.53$\times$ \\
CUTS+ ($L=3$) & 0.571 / 1.31$\times$ & 0.563 / 1.27$\times$ & 0.577 / 1.33$\times$ & 0.586 / 1.36$\times$ \\
SCDM & 0.566 / 1.96$\times$ & \textbf{0.566} / \textbf{1.96}$\times$ & \textbf{0.605} / \textbf{2.47}$\times$ & \textbf{0.645} / \textbf{4.01}$\times$ \\
\bottomrule
\end{tabular}
\end{table}

Table~\ref{tab:voxel_fmri_fixed_budget} shows that SCDM improves steadily with the budget and reaches 0.645 AUROC and $4.01\times$ enrichment at 800 seconds, compared with 0.586 and $1.36\times$ for CUTS+, which is designed for high-dimensional causal discovery, and 0.541 and $1.53\times$ for UnCLe. At 200 seconds UnCLe is slightly ahead on both metrics and CUTS+ on AUROC; from 400 seconds on, SCDM leads on both metrics, and its margin grows with the budget. At this scale we use the structural readout, because the source-isolated readout requires $O(d^2)$ evaluations of the trained predictor (Appendix~\ref{app:baseline_hparams}). Appendix~\ref{app:voxel_fmri_results} plots the same results and explains why the other baselines do not fit this fixed-budget protocol.

\subsection{Real-World LargeST Scalability}
\label{sec:largest_scalability}

We evaluate SCDM on LargeST as a real-world traffic benchmark for scalability, forecasting consistency, and qualitative structural plausibility. Since LargeST lacks a complete directed causal ground-truth graph, we do not claim exact causal recovery or report causal AUROC/AUPRC.

Figure~\ref{fig:largest_gla_causal_visualization} shows that the high-scoring directed relations are spatially coherent: as training proceeds, they concentrate on nearby freeway sensors instead of arbitrary long-range pairs. This is a spatial-plausibility check, not evidence that individual edges are causal.

\begin{figure*}[t]
    \centering
    \includegraphics[width=\textwidth]{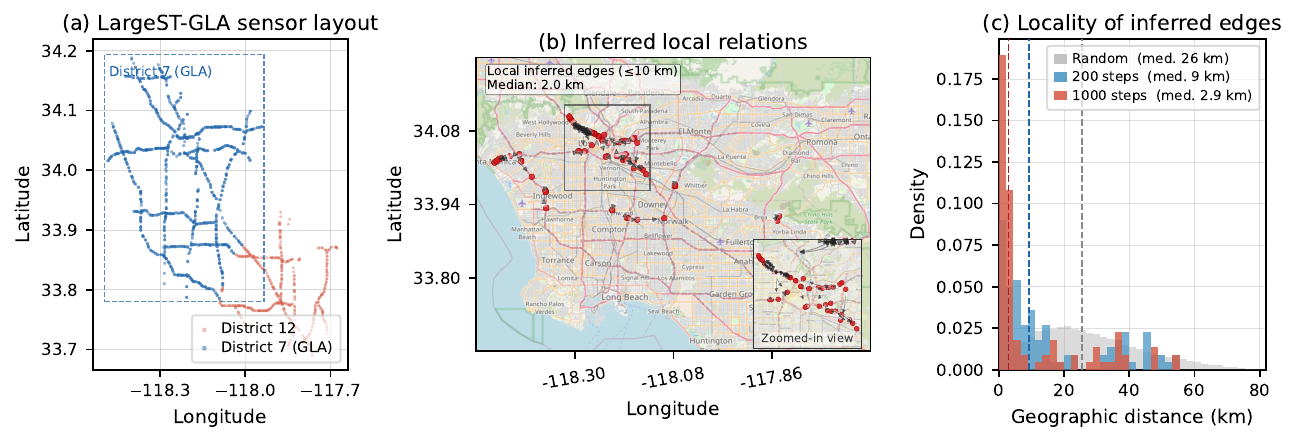}
    \caption{Spatial-plausibility diagnostic on LargeST-GLA. (a) Locations of the sensors used in this experiment; this panel shows geography only. (b) Representative high-scoring directed relations from SCDM. Each arrow points from the inferred source sensor to the inferred target sensor; only relations shorter than 10\,km are drawn to reduce clutter, and arrow curvature is for readability only. (c) Distribution of the geographic distance between the two sensors of each top-scoring relation. The median distance falls from about 26\,km for random sensor pairs to 9\,km after 200 training steps and 2.9\,km after 1{,}000 steps, so high-scoring relations become more local as training proceeds.}
\label{fig:largest_gla_causal_visualization}
\end{figure*}

\begin{table}[t]
\centering
\small
\caption{Feasibility on LargeST-GLA-3000 (3{,}000 sensors, $8{,}997{,}000$ candidate pairs). One A100 80\,GB GPU per run, 80-minute limit, three seeds. A run counts as completed only if it writes directed scores for all candidate pairs. Time to outcome is reported as mean $\pm$ standard deviation (SD) across three seeds.}
\label{tab:gla3000_feasibility}
\begin{tabular}{@{}lccc@{}}
\toprule
Method & Outcome & Time to outcome (s) & Peak GPU (GB) \\
\midrule
SCDM & Completed & $56.96 \pm 4.23$ & 6.9 \\
UnCLe & Out of memory & $349.25 \pm 13.30$ & 15.6 \\
CUTS+ & Time limit & $>4800$ & 72.7 \\
JRNGC & Time limit & $>4800$ & 14.3 \\
cMLP & Time limit & $>4800$ & 38.1 \\
\bottomrule
\end{tabular}
\end{table}

Table~\ref{tab:gla3000_feasibility} tests whether each method can run on the 3{,}000-sensor subset, which has nearly nine million candidate pairs.
SCDM finishes training and graph extraction in 56.96\,s with 6.9\,GB peak GPU memory. UnCLe writes an intermediate checkpoint after its first training phase and then runs out of memory in its joint phase; CUTS+, JRNGC~\citep{zhou2024jacobian}, and cMLP start normally but do not write a completed score matrix within 80 minutes. This test measures whether a method can run at this scale, not causal accuracy. Appendix~\ref{app:largest_forecasting_results} adds a graph-conditioned forecasting diagnostic on LargeST-GLA-1500 (Table~\ref{tab:largest_graph_forecasting_aux}).

\subsection{Ablation Study}
\label{sec:ablation}

We run two ablations, one for each training component described in Section~\ref{sec:meta-method}.

\paragraph{Target-task subsampling.} All variants use the same full method, dataset, optimizer, number of outer updates, and graph-score evaluation; only the number of target tasks $B$ sampled per meta-episode changes. We report training time only, because the post-training readout is identical across variants and does not measure the cost of meta-training. A smaller per-update cost could in principle be offset by needing more updates, so we fix the number of outer updates and compare total training time. Table~\ref{tab:fmri_task_batch_ablation} shows that $B=32$ cuts total training time by $3.1\times$ relative to all-target updates ($B=100$) and also gives the best recovery.

\begin{table}[t]
\centering
\caption{Task-batch-size ablation on the 100-variable simulated fMRI benchmark.
All variants use the full method with the structural readout and differ only in the number of target tasks per
meta-episode. Results are mean $\pm$ 95\% CI over five seeds; runtime reports
training time only.}
\label{tab:fmri_task_batch_ablation}
\small
\setlength{\tabcolsep}{3.2pt}
\resizebox{\columnwidth}{!}{%
\begin{tabular}{@{}lcccccc@{}}
\toprule
\textbf{Task batch $B$}
& \textbf{AUROC $\uparrow$}
& \textbf{AUPRC $\uparrow$}
& \textbf{F1 $\uparrow$}
& \textbf{SHD $\downarrow$}
& \textbf{Train (s) $\downarrow$}
& \textbf{Speedup $\uparrow$} \\
\midrule
10
& $0.819 \pm 0.007$
& $0.656 \pm 0.017$
& $0.541 \pm 0.030$
& $841.2 \pm 30.9$
& $\mathbf{226.9 \pm 2.5}$
& $9.0{\times}$ \\
32
& $\mathbf{0.867 \pm 0.006}$
& $\mathbf{0.728 \pm 0.019}$
& $\mathbf{0.675 \pm 0.025}$
& $\mathbf{770.6 \pm 22.2}$
& $668.4 \pm 7.2$
& $3.1{\times}$ \\
64
& $0.822 \pm 0.007$
& $0.680 \pm 0.012$
& $0.588 \pm 0.015$
& $785.2 \pm 14.9$
& $1301.0 \pm 24.8$
& $1.6{\times}$ \\
100
& $0.747 \pm 0.008$
& $0.552 \pm 0.015$
& $0.392 \pm 0.033$
& $1020.8 \pm 29.9$
& $2049.9 \pm 13.3$
& $1.0{\times}$ \\
\bottomrule
\end{tabular}
}
\end{table}

\paragraph{Support/query meta-objective.} Subsampling could also be combined with ordinary multi-task training, so the ablation above does not show what the support/query objective adds. We hold out a subset of target variables from training and train three variants with the same architecture, data, initialization, and budget: full SCDM, an episodic variant without inner adaptation ($K=0$), and pooled multi-task training on sampled targets (Joint MTL). All three then adapt to each held-out target with the same test-time protocol (Appendix~\ref{app:heldout}). SCDM reaches a post-adaptation query MSE of 3.782, which is 12.1\% lower than Joint MTL (4.305) and 31.1\% lower than $K=0$ (5.487). Because sampling, data, and test-time adaptation are identical, the gap comes from the training objective. The shared matrix $W$ also matters for recovery: removing it lowers AUPRC from 0.728 to 0.390 (Appendix~\ref{app:component_ablation_schaefer100}).

\paragraph{Readout choice.} Applied to the same trained models (Table~\ref{tab:readout_ablation}), the structural readout ranks edges better on simulated fMRI (AUPRC 0.728 vs.\ 0.693), so most of the ranking information is already stored in $W$, while the source-isolated readout is better on simulated EEG (0.476 vs.\ 0.426). 

\paragraph{Stress tests.} In NC8-Collinear (Appendix~\ref{app:collinear}), graph-level AUPRC stays between 0.741 and 0.965 as a non-parent becomes correlated with a true parent (correlation up to 0.89), but the non-edge outranks the true edge in 5 of 30 runs. With four of the 100 variables in simulated fMRI and EEG withheld, a latent-augmented variant keeps observed-block AUPRC at 0.692 and 0.470 (Appendix~\ref{app:latent_results}).

\section{Conclusion}
\label{sec:conclusion}

We introduced SCDM, a shared-parameter meta-learning framework for scalable causal discovery in nonlinear temporal systems. By treating target-variable prediction as a collection of related tasks, SCDM learns a shared temporal predictor and causal-strength matrix while updating only a subset of target tasks per meta-episode. Ablations show that target-task subsampling lowers training cost without hurting recovery and that the support/query objective gives shared parameters that adapt better than pooled multi-task training. SCDM is competitive on controlled benchmarks and gives a better runtime--accuracy tradeoff at 1{,}000 and 3{,}000 variables. SCDM remains an observational method with clear limits. Its recovery result assumes no hidden common causes, a correct lag order, and no instantaneous effects; a hidden driver of several observed variables can create spurious edges, and our latent-variable experiments (Appendix~\ref{app:latent_results}) are a controlled diagnostic rather than a general solution. Individual edge scores can also move to correlated non-parents when sources are nearly redundant (Appendix~\ref{app:collinear}). Future work will study adaptive lag selection, adaptive task sampling, and explicit modeling of hidden confounders.

\begin{ack}
We thank the anonymous reviewers and the area chair for their constructive comments and suggestions, which improved the clarity of the method description, the design of the ablation studies, and the discussion of assumptions and limitations. Computing resources were provided by the University of Tennessee at Chattanooga. The authors declare no competing interests.
\end{ack}

\bibliographystyle{plainnat}
\bibliography{ref}

\newpage

\appendix

\section*{Notation}
\begin{table}[h]
\centering
\small
\caption{Notation used throughout the paper.}
\label{tab:scdm_notation}
\begin{tabularx}{\linewidth}{@{}lX@{}}
\toprule
Symbol & Definition \\
\midrule
$T$ & Number of time points \\
$d$ & Number of observed variables, which is also the number of target tasks \\
$p$ & Observed-history order: the input at time $t$ is $x_{t-1},\ldots,x_{t-p}$ \\
$\ell$ & Lag index, $\ell=1,\ldots,p$ \\
$\loss(\cdot,\cdot)$ & Pointwise prediction loss (squared error in all experiments) \\
$\theta$ & Shared temporal predictor parameters \\
$W\in\mathbb{R}^{d\times d}$ & Shared causal-strength matrix; $W_{j,i}$ is the strength from source $j$ to target $i$, and column $W[:,i]$ gates all sources of target $i$ \\
$r$ & Number of components per edge in $W$ ($r=1$ unless stated) \\
$\widetilde H^{(i)}_t$ & Weighted lagged history of target $i$, with entries $W_{j,i}x^{(j)}_{t-\ell}$ \\
$B$ & Number of target tasks sampled per meta-episode \\
$K$ & Number of inner (support) adaptation steps \\
$N$ & Number of outer meta-updates \\
$L_{\mathrm{sup}},L_{\mathrm{qry}}$ & Lengths of the support and query windows \\
$\alpha,\beta$ & Inner and outer learning rates \\
$\lambda_W$ & Sparsity coefficient on $W[:,i]$, shared by the support and query losses \\
$\mathrm{Pa}(i)$ & Set of true lagged parents of target $i$ \\
$S^\star_{ji}$ & Population ablation score: increase in the optimal risk of target $i$ when source $j$ is removed \\
$\Delta_{\min}$ & Smallest $S^\star_{ji}$ over true edges \\
$s_{j\to i}$ & Implemented edge score from source $j$ to target $i$ \\
$\epsilon_{\mathrm{stat}},\epsilon_{\mathrm{meta}},\epsilon_{\mathrm{readout}}$ & Statistical, meta-optimization, and readout errors in Theorem~\ref{thm:finite_error_graph_recovery} \\
\bottomrule
\end{tabularx}
\end{table}

\section{Theoretical Analysis}
\label{app:theory}

This appendix provides the proof of Theorem~\ref{thm:finite_error_graph_recovery}. The analysis is intentionally restricted to the fully observed non-latent setting. It does not claim recovery under hidden confounding, instantaneous effects, incorrect lag specification, or misspecified predictors. The purpose of the result is to formalize a finite-error recovery principle: if the learned score is uniformly close to a population risk-gap score, then thresholding recovers the delayed causal graph.

\subsection{Population Target and Assumptions}
\label{app:theory_assumptions}

Let \(X_t=(X_t^{(1)},\ldots,X_t^{(d)})\) be a stationary multivariate temporal process and let \(H_t=[X_{t-1},\ldots,X_{t-p}]\in\mathbb{R}^{d\times p}\) denote the lagged history. For each target \(i\), the parent set \(\mathrm{Pa}(i)\subseteq\{1,\ldots,d\}\setminus\{i\}\) contains source variables whose lagged histories directly influence \(X_t^{(i)}\). For the pointwise loss \(\loss(\cdot,\cdot)\), define
\[
R_i(f)=\mathbb{E}[\loss(f(H_t),X_t^{(i)})],
\qquad
R_i^\star=\inf_{f\in\mathcal{F}_i}R_i(f).
\]
For a candidate source \(j\), let \(H_{t,-j}\) be the lagged history with all lags of source \(j\) removed, and define the restricted optimal risk
\[
R_{i,-j}^\star=\inf_{f\in\mathcal{F}_{i,-j}}
\mathbb{E}[\loss(f(H_{t,-j}),X_t^{(i)})].
\]
The population ablation score is
\[
S_{ji}^\star=R_{i,-j}^\star-R_i^\star.
\]

We group the assumptions into four categories: identifiability, statistical concentration, meta-optimization, and readout fidelity.

\begin{assumption}[Delayed-system identifiability]
\label{assump:delayed_identifiability}
The process is fully observed and follows a delayed temporal causal graph with no instantaneous effects.
The learner's lag order covers the true maximum causal lag.
Conditional on the lagged histories of its parents, each target is independent of non-parent lagged histories.
The predictor class is rich enough to represent the population-optimal delayed conditional predictor.
Finally, the population ablation score has a positive margin: there exists \(\Delta_{\min}>0\) such that
\(S^\star_{ji}\ge \Delta_{\min}\) for every \(j\in\mathrm{Pa}(i)\), and \(S^\star_{ji}=0\) for every
\(j\notin\mathrm{Pa}(i)\).
\end{assumption}

\begin{assumption}[Uniform concentration under temporal dependence]
\label{assump:uniform_concentration}
The temporal process is stationary and sufficiently mixing.
The loss-composed full and source-restricted predictor classes have finite sequential complexity
\(\mathcal{C}(\mathcal{F})\), yielding a uniform concentration bound with effective sample size
\(T_{\mathrm{eff}}\).
\end{assumption}

\begin{assumption}[Meta-optimization residual]
\label{assump:meta_residual}
After \(N\) outer iterations, task batch size \(B\), and \(K\) inner adaptation steps, the SCDM
meta-learned empirical score \(\widehat S^{\mathrm{meta}}_{ji,T,N}\) satisfies
\[
\max_{i,j}
\left|
\widehat S^{\mathrm{meta}}_{ji,T,N}
-
\widehat S^{\mathrm{abl}}_{ji,T}
\right|
\le
\epsilon_{\mathrm{meta}}(N,B,K),
\]
where \(\widehat S^{\mathrm{abl}}_{ji,T}\) is the ideal empirical ablation score.
This term captures finite episodic meta-training, target-task subsampling, finite support adaptation,
and the first-order meta-update.
We use existing MAML and first-order MAML convergence analyses only to motivate this residual form;
we do not claim a new meta-optimization rate.
\end{assumption}

\begin{assumption}[Readout residual]
\label{assump:readout_residual}
Let \(\widetilde S_{ji,T,N}\) be the implemented post-training source-isolated score.
Its discrepancy from the meta-learned empirical ablation score is bounded by
\[
\max_{i,j}
\left|
\widetilde S_{ji,T,N}
-
\widehat S^{\mathrm{meta}}_{ji,T,N}
\right|
\le
\epsilon_{\mathrm{readout}}(T,N).
\]
This residual records the approximation introduced by using a scalable post-training readout instead
of retraining a restricted predictor for every ordered pair.
\end{assumption}

\subsection{Statistical, Meta-optimization, and Readout Errors}
\label{app:theory_error_decomp}

We first control the gap between the ideal empirical ablation score and the population ablation score. Let $\widehat{R}_{i,T}(f)=T'^{-1}\sum_{t=p+1}^{T}\loss(f(H_t),X_t^{(i)})$, where $T'=T-p$. Similarly define $\widehat{R}_{i,-j,T}(f)=T'^{-1}\sum_{t=p+1}^{T}\loss(f(H_{t,-j}),X_t^{(i)})$. The ideal empirical risks are $\widehat{R}_{i,T}^{\star}=\inf_{f\in\mathcal{F}_i}\widehat{R}_{i,T}(f)$ and $\widehat{R}_{i,-j,T}^{\star}=\inf_{f\in\mathcal{F}_{i,-j}}\widehat{R}_{i,-j,T}(f)$, and the ideal empirical ablation score is $\widehat{S}^{\mathrm{abl}}_{ji,T}=\widehat{R}_{i,-j,T}^{\star}-\widehat{R}_{i,T}^{\star}$.

\begin{lemma}[Uniform convergence of ablation scores]
\label{lem:uniform_ablation}
Under Assumption~\ref{assump:uniform_concentration}, with probability at least \(1-\delta\),
\[
\max_{i,j}\left|\widehat{S}^{\mathrm{abl}}_{ji,T}-S_{ji}^{\star}\right|\le \epsilon_{\mathrm{stat}}(T,d,\delta),\qquad
\epsilon_{\mathrm{stat}}(T,d,\delta)=C\sqrt{\frac{\mathcal{C}(\mathcal{F})+\log(2d^2/\delta)}{T_{\mathrm{eff}}}} .
\]
Here $C$ is a constant depending on the loss scale and mixing constants. In the simplified bounded-complexity and geometrically mixing case, this reduces to $\epsilon_{\mathrm{stat}}(T,d,\delta)=O(\sqrt{\log(d^2/\delta)/T})$.
\end{lemma}

\begin{proof}
For any fixed pair $(j,i)$,
\[
\left|\widehat{S}^{\mathrm{abl}}_{ji,T}-S_{ji}^{\star}\right|
\le
\left|\widehat{R}^{\star}_{i,-j,T}-R^{\star}_{i,-j}\right|
+
\left|\widehat{R}^{\star}_{i,T}-R^{\star}_{i}\right|.
\]
The first term is controlled by the uniform deviation over $\mathcal{F}_{i,-j}$:
\[
\left|\widehat{R}^{\star}_{i,-j,T}-R^{\star}_{i,-j}\right|
\le
\sup_{f\in\mathcal{F}_{i,-j}}
\left|\widehat{R}_{i,-j,T}(f)-R_{i,-j}(f)\right|.
\]
The same argument gives $\left|\widehat{R}^{\star}_{i,T}-R^{\star}_{i}\right|\le \sup_{f\in\mathcal{F}_{i}}\left|\widehat{R}_{i,T}(f)-R_i(f)\right|$. By the mixing concentration assumption and the finite complexity of the
loss-composed classes, each uniform deviation is bounded by a term of order
\(\sqrt{(\mathcal{C}(\mathcal{F})+\log(1/\delta'))/T_{\mathrm{eff}}}\).
Applying a union bound over at most \(d(d-1)\le d^2\) ordered pairs and
absorbing constants yields the stated bound. The simplified rate follows under
bounded complexity and geometric mixing.
\end{proof}

The statistical bound above concerns the ideal empirical ablation score. The implemented method differs from this ideal object in two ways. First, the temporal predictor and causal-strength matrix are obtained by a finite number of meta-training iterations. Since each target variable is treated as a task, the shared initialization and shared temporal representation are optimized through support/query episodes. Finite $N$, finite task batch size $B$, finite inner adaptation $K$, and the first-order approximation all contribute to the residual
\[
\epsilon_{\mathrm{meta}}(N,B,K)=\max_{i,j}\left|\widehat{S}^{\mathrm{meta}}_{ji,T,N}-\widehat{S}^{\mathrm{abl}}_{ji,T}\right|.
\]
This term captures how closely the learned meta-predictor approximates the empirical full and restricted risk minimizers. Existing MAML convergence results support the view that such optimization residuals can decrease under standard smoothness, bounded-gradient, and step-size conditions~\citep{fallah2020convergence,chayti2024first}. We use this literature only to motivate the residual form; our graph-recovery theorem does not require claiming a new rate for the particular first-order implementation.

Second, the implemented graph extraction uses a post-training perturbation readout rather than explicitly retraining a restricted model for every ordered pair. We therefore define
\[
\epsilon_{\mathrm{readout}}(T,N)=\max_{i,j}\left|\widetilde{S}_{ji,T,N}-\widehat{S}^{\mathrm{meta}}_{ji,T,N}\right|.
\]
This term measures how faithfully the perturbation procedure approximates the ideal source-ablation score. When the trained predictor is well specified and the perturbation operation faithfully removes source-specific historical information, this discrepancy vanishes in the infinite-data and fully optimized limit. In finite samples, it is a practical residual controlled jointly by the quality of statistical estimation, meta-optimization, and score normalization.

\begin{lemma}[Error decomposition]
\label{lem:error_decomp}
Under Assumptions~\ref{assump:uniform_concentration},
\ref{assump:meta_residual}, and~\ref{assump:readout_residual},
with probability at least \(1-\delta\),
\[
\max_{i,j}\left|\widetilde{S}_{ji,T,N}-S_{ji}^{\star}\right|
\le
\epsilon_{\mathrm{stat}}(T,d,\delta)+\epsilon_{\mathrm{meta}}(N,B,K)+\epsilon_{\mathrm{readout}}(T,N).
\]
\end{lemma}

\begin{proof}
For each ordered pair $(j,i)$, add and subtract $\widehat{S}^{\mathrm{meta}}_{ji,T,N}$ and $\widehat{S}^{\mathrm{abl}}_{ji,T}$:
\[
\left|\widetilde{S}_{ji,T,N}-S_{ji}^{\star}\right|
\le
\left|\widetilde{S}_{ji,T,N}-\widehat{S}^{\mathrm{meta}}_{ji,T,N}\right|
+
\left|\widehat{S}^{\mathrm{meta}}_{ji,T,N}-\widehat{S}^{\mathrm{abl}}_{ji,T}\right|
+
\left|\widehat{S}^{\mathrm{abl}}_{ji,T}-S_{ji}^{\star}\right|.
\]
Taking the maximum over all \((j,i)\) and applying
Lemma~\ref{lem:uniform_ablation},
Assumption~\ref{assump:meta_residual}, and
Assumption~\ref{assump:readout_residual} gives the result.
\end{proof}

\subsection{Proof of Finite-Error Graph Recovery}
\label{app:theory_recovery_proof}

\begin{theorem}[Finite-error graph recovery]
\label{thm:appendix_finite_error_recovery}
Suppose Assumptions~\ref{assump:delayed_identifiability},
\ref{assump:uniform_concentration}, \ref{assump:meta_residual},
and~\ref{assump:readout_residual} hold. Let $\epsilon_{\mathrm{tot}}=\epsilon_{\mathrm{stat}}(T,d,\delta)+\epsilon_{\mathrm{meta}}(N,B,K)+\epsilon_{\mathrm{readout}}(T,N)$. If $\epsilon_{\mathrm{tot}}<\Delta_{\min}/2$, then thresholding the practical score at any
\[
\gamma\in\left(\epsilon_{\mathrm{tot}},\Delta_{\min}-\epsilon_{\mathrm{tot}}\right)
\]
recovers the true graph with probability at least $1-\delta$: $\Pr(\widehat{A}_{T,N}=A^\star)\ge 1-\delta$. In particular, $\gamma=\Delta_{\min}/2$ is valid.
\end{theorem}

\begin{proof}
By Lemma~\ref{lem:error_decomp}, with probability at least $1-\delta$, $\max_{i,j}|\widetilde{S}_{ji,T,N}-S_{ji}^{\star}|\le\epsilon_{\mathrm{tot}}$. Consider a true parent $j\in\mathrm{Pa}(i)$. By
Assumption~\ref{assump:delayed_identifiability},
\(S^\star_{ji}\ge\Delta_{\min}\), and therefore $\widetilde{S}_{ji,T,N}\ge S_{ji}^{\star}-\epsilon_{\mathrm{tot}}\ge\Delta_{\min}-\epsilon_{\mathrm{tot}}>\gamma$, so the edge $j\to i$ is correctly included. Now consider a non-parent $j\notin\mathrm{Pa}(i)$. By
Assumption~\ref{assump:delayed_identifiability},
\(S^\star_{ji}=0\), and therefore $\widetilde{S}_{ji,T,N}\le S_{ji}^{\star}+\epsilon_{\mathrm{tot}}=\epsilon_{\mathrm{tot}}<\gamma$, so the edge $j\to i$ is correctly excluded. Since every ordered pair is classified correctly on the same high-probability event, the full graph is recovered with probability at least $1-\delta$.
\end{proof}

\subsection{Asymptotic Recovery and Interpretation}
\label{app:theory_interpretation}

The finite-error theorem immediately yields an asymptotic consistency statement.

\begin{corollary}[Asymptotic full-graph recovery]
\label{cor:asymptotic_recovery}
Under the assumptions of Theorem~\ref{thm:appendix_finite_error_recovery}, if $\epsilon_{\mathrm{stat}}(T,d,\delta_T)\to 0$, $\epsilon_{\mathrm{meta}}(N,B,K)\to 0$, and $\epsilon_{\mathrm{readout}}(T,N)\to 0$ as $T,N\to\infty$ with $\delta_T\to 0$, then $\Pr(\widehat{A}_{T,N}=A^\star)\to 1$.
\end{corollary}

\begin{proof}
For sufficiently large $T$ and $N$, $\epsilon_{\mathrm{stat}}(T,d,\delta_T)+\epsilon_{\mathrm{meta}}(N,B,K)+\epsilon_{\mathrm{readout}}(T,N)<\Delta_{\min}/2$. The result follows directly from Theorem~\ref{thm:appendix_finite_error_recovery}.
\end{proof}

A concrete sufficient condition can be obtained from the statistical bound above. For example, it is enough that
\[
T_{\mathrm{eff}}\ge C\frac{\mathcal{C}(\mathcal{F})+\log(2d^2/\delta)}{\Delta_{\min}^2},
\qquad
\epsilon_{\mathrm{meta}}(N,B,K)+\epsilon_{\mathrm{readout}}(T,N)\le\frac{\Delta_{\min}}{4}.
\]
This condition should be interpreted as a sufficient recovery regime rather than a sharp sample-complexity characterization. It shows that graph recovery is controlled jointly by temporal sample size, function-class complexity, graph dimension through a union bound over $d^2$ candidate edges, and the quality of meta-optimization and perturbation readout.

The theorem is stated in terms of an ideal ablation risk-gap target because this quantity has a direct population interpretation: it measures how much optimal prediction risk increases when the source history of variable $j$ is removed. The implemented method avoids retraining a separate restricted predictor for every ordered pair $(j,i)$. Instead, after training a shared temporal predictor, it computes a post-training perturbation score that measures how the target prediction changes when source-specific historical information is removed or neutralized.

This distinction is important in high dimensions. Explicitly fitting all restricted models would require $O(d^2)$ model refits, which is incompatible with voxel-level or traffic-scale systems. The perturbation readout is therefore used as a scalable approximation to the same ablation principle. The residual $\epsilon_{\mathrm{readout}}$ explicitly records this approximation gap. In the infinite-data, well-specified, and fully optimized limit, if the perturbation operation faithfully removes the same source-specific information as the ideal ablation operator, then $\epsilon_{\mathrm{readout}}\to 0$. In finite samples, the residual depends on predictor quality, optimization quality, score normalization, and the fidelity of the perturbation operator.

Finally, the identifiability component of the theorem comes from the delayed causal system and the population risk margin, not from meta-learning itself. Meta-learning enters through the finite-error term $\epsilon_{\mathrm{meta}}(N,B,K)$. In the proposed framework, each target variable defines a task, and all tasks share the same temporal predictor and causal-strength parameterization. The support/query split encourages the shared parameters to adapt efficiently to target-specific prediction tasks, while the outer update aggregates evidence across a mini-batch of target variables. Compared with fitting independent target-specific models, this shared episodic training can reduce adaptation error and optimization error in finite samples. The theory captures this effect through the condition that $\epsilon_{\mathrm{meta}}$ becomes small enough relative to the population margin $\Delta_{\min}$.

Thus, the theorem does not claim that meta-learning alone identifies causal structure. Rather, under standard delayed-system identifiability conditions, meta-learning provides a scalable optimization mechanism for approaching the population risk-gap score in high-dimensional temporal systems.

\section{Experimental Details and Supplementary Results}
\label{app:experiments}

This appendix consolidates the implementation details and supplementary results supporting the empirical claims in the main text. We place baseline configurations first, followed by SCDM training and readout details, and then the supplementary dataset protocols and experimental results.

\subsection{Baseline and Implementation Details}
\label{app:baseline_hparams}
\label{app:hyperparameters}

This subsection reports the baseline configurations, SCDM hyperparameters, source-isolated readout protocol, thresholding rules, and compute setup used in the experiments. All methods are evaluated on the same train/validation/test split whenever applicable, and self-edges are excluded from graph-level evaluation unless otherwise stated.

\paragraph{Baseline methods.}
We report the baseline hyperparameter settings used in the controlled benchmark comparison. 
For NC8, we follow the established benchmark configuration used in prior work whenever available. 
For the simulated fMRI and simulated EEG benchmarks, lag-based baselines are configured to cover the corresponding data-generating temporal dependencies whenever supported by the public implementation. 
Unless otherwise stated, remaining hyperparameters follow the default public implementations.

\paragraph{UnCLe.}
UnCLe~\citep{bi2025uncle} is available at \url{https://github.com/etigerstudio/uncle-causal-discovery}. On NC8, we follow the standard benchmark configuration with lag 1, kernel size 8, 6 TCN blocks, 20 filters, a $1000+2000$ epoch training schedule, and learning rate $3\times10^{-4}$. For simulated fMRI and simulated EEG, we set lag 3 and kernel size 5 to match the data-generating temporal dependency. Both neuroimaging-inspired benchmarks use one hidden layer with hidden size 50, encoder and decoder hidden layers $[50]$, prediction length 1, a $10+10$ epoch training schedule, and learning rate $10^{-4}$.

\paragraph{PCMCI.}
For PCMCI~\citep{runge2019detecting}, we use partial correlation as the conditional independence test unless otherwise specified. The maximum lag is set to the benchmark lag for NC8 and to the known simulation lag for simulated fMRI and simulated EEG. Since PCMCI is a constraint-based method and does not require neural-network training, its main settings are the conditional independence test, maximum lag, and significance threshold rather than epochs or learning rates.

\paragraph{cMLP.}
cMLP~\citep{tank2021neural} is available at \url{https://github.com/iancovert/Neural-GC}. On NC8, we use lag 16, one hidden layer, 1000 epochs, and learning rate $5\times10^{-3}$. For simulated fMRI and simulated EEG, we use lag 3, hidden size $[20]$, 3000 maximum iterations, and learning rate $5\times10^{-2}$. The target-specific cMLP architecture uses $\mathrm{Conv1d}(p\!\rightarrow\!20,k=5)$ followed by $\mathrm{Conv1d}(20\!\rightarrow\!1,k=1)$, where $p$ is the input lag order.

\paragraph{TCDF.}
TCDF~\citep{nauta2019causal} is available at \url{https://github.com/M-Nauta/TCDF}. On NC8, we use kernel size 16, one hidden layer, 1000 epochs, and learning rate $5\times10^{-3}$. For simulated fMRI and simulated EEG, we use kernel size 4 and dilation coefficient 4 so that the effective temporal range covers lag-3 dependencies. Both neuroimaging-inspired benchmarks use one effective hidden layer, 1000 epochs, and learning rate $10^{-2}$. In the implementation, we use the ADDSTCN architecture and set the maximum explicit delay to 5 for the neuroimaging-inspired benchmarks.

\paragraph{GVAR.}
GVAR~\citep{marcinkevivcs2021interpretable} is available at \url{https://github.com/i6092467/GVAR}. On NC8, we use lag order 16, one hidden layer, 1000 epochs, and learning rate $10^{-4}$. For simulated fMRI and simulated EEG, we use lag order 3, one hidden layer with width 50, 1000 epochs, and learning rate $10^{-4}$. The model uses one SENNGC coefficient network per lag.

\paragraph{CUTS+.}
CUTS+~\citep{cheng2024cuts+} is available at \url{https://github.com/jarrycyx/UNN/tree/main/CUTS_Plus}. We use learning rate $10^{-3}$, 64 training epochs, and maximum number of groups 32. For NC8, the temporal range is aligned with the benchmark lag setting $L=16$. For simulated fMRI and simulated EEG, the temporal range is set to cover lag-3 dependencies. We tune the regularization coefficient over $\lambda\in\{0.1,0.05,0.01,0.005\}$.

\begin{table}[h]
\centering
\small
\caption{Temporal settings of each method. These quantities are not directly comparable: SCDM's $p$ and the observed lag orders set how many past observations enter the model, UnCLe's lag acts on latent representations after a convolutional encoder, and convolution kernels set a receptive field. The true maximum lag is 16 for NC8, 5 for simulated EEG, and 3 for simulated fMRI.}
\label{tab:baseline_temporal}
\begin{tabularx}{\linewidth}{@{}llcccX@{}}
\toprule
Method & Temporal setting & NC8 & Sim.\ EEG & Sim.\ fMRI & How chosen \\
\midrule
SCDM & observed-history order $p$ & 16 & 5 & 3 & True maximum lag \\
UnCLe & latent AR lag / kernel & 1 / 8 & 3 / 5 & 3 / 5 & NC8: UnCLe benchmark configuration; others: cover the true lag \\
cMLP & observed lag order & 16 & 3 & 3 & NC8: UnCLe benchmark configuration; others: lag 3 \\
GVAR & observed lag order & 16 & 3 & 3 & As for cMLP \\
TCDF & kernel / dilation / max delay & 16 / -- / -- & 4 / 4 / 5 & 4 / 4 / 5 & Receptive field covers the true lag \\
PCMCI & maximum lag $\tau_{\max}$ & 16 & 5 & 3 & True maximum lag \\
CUTS+ & temporal range $L$ & 16 & 3 & 3 & NC8: 16; others: 3 \\
\bottomrule
\end{tabularx}
\end{table}

\paragraph{Evaluation consistency.}
All methods are evaluated on the same train/validation/test split whenever applicable. Continuous directed scores are evaluated using AUROC and AUPRC. For threshold-dependent metrics such as F1 and SHD, the threshold is selected on the validation split and applied once to the test split. Unless otherwise stated, self-edges are excluded from graph-level evaluation because they represent autoregressive self-dependencies rather than cross-variable causal discovery.

\paragraph{SCDM.}
We report all hyperparameter choices for SCDM.
Unless noted otherwise, the settings in this section were used for every dataset.

\noindent\textit{Model architecture.} The prediction network $f_\theta$ is a two-hidden-layer MLP with widths $dpr\to128\to64\to1$ and ReLU activations, where $d$ is the number of variables, $p$ the observed-history order, and $r$ the number of components per edge in $W$. We use $r=1$ (scalar edge strength) for all datasets except LargeST-GLA-1500 ($r=16$) and LargeST-GLA-3000 ($r=8$).

\noindent\textit{Causal-strength matrix initialization.}
Entries of $\mathbf{W}$ are drawn independently from
$\mathrm{Uniform}(0.0,\,0.3)$ for all datasets except LargeST-GLA-1500 and
LargeST-GLA-3000, where the range is $\mathrm{Uniform}(0.0,\,0.05)$.
Self-edges are permanently masked to zero before and during training
(\texttt{allow\_self\_loops=false}).

\noindent\textit{Regularization.} For a sampled target $i$, both the support and the query loss penalize only that target's column $W[:,i]$, so each task regularizes only the incoming edges it can inform. For $r=1$ the penalty is the element-wise $\ell_1$ norm; for $r>1$ it is a group lasso, the sum over sources of the $\ell_2$ norm of each edge's $r$ components. The coefficient $\lambda_W$ is set per dataset (Table~\ref{tab:hparam_scdm}).

\noindent\textit{Input normalization.}
Time series are z-score normalized using statistics estimated on the training
split, with the same transform applied to validation and test.
The simulated fMRI dataset is provided in pre-normalized form and is used
without further scaling.

\noindent\textit{Meta-learning.} SCDM uses first-order MAML with $K=3$ inner gradient steps per task. The inner learning rate is $\alpha_\theta=0.01$ for the predictor and $\alpha_W=0.001$ for $W$; the smaller rate keeps a single task from moving the shared graph too far during adaptation. The outer optimizer is Adam with learning rate $\beta=0.001$ and weight decay $10^{-5}$, and gradients are clipped to unit $\ell_2$ norm before each outer update.

\noindent\textit{Episode structure.} Each outer update samples $B$ target tasks uniformly at random. For each task, a contiguous support window of length $L_{\mathrm{sup}}$ and a contiguous query window of length $L_{\mathrm{qry}}$ are drawn from the training split. Table~\ref{tab:hparam_scdm} lists $B$, $L_{\mathrm{sup}}$, $L_{\mathrm{qry}}$, $p$, $\lambda_W$, and the number of outer updates $N$ for each dataset.

\begin{table}[t]
\centering
\caption{Dataset-specific hyperparameters for SCDM. $d$: number of observed variables; $p$: observed-history order; $B$: target tasks per outer update; $L_{\mathrm{sup}}$/$L_{\mathrm{qry}}$: support/query window length; $\lambda_W$: sparsity coefficient (same in support and query losses); $N$: number of outer updates; Split: train/validation/test fraction.}
\label{tab:hparam_scdm}
\small
\begin{tabular}{lrrrrrrrr}
\toprule
Dataset & $d$ & $p$ & $B$ & $L_{\mathrm{sup}}$ & $L_{\mathrm{qry}}$ & $\lambda_W$ & $N$ & Split \\
\midrule
NC8
  & 8 & 16 & 8
  & 64 & 64
  & 0.05 & 1{,}000 & 70/15/15 \\
Simulated fMRI
  & 100 & 3 & 32
  & 64 & 64
  & 0.05 & 5{,}000 & 70/15/15 \\
Simulated EEG
  & 100 & 5 & 32
  & 64 & 64
  & 0.05 & 3{,}000 & 70/15/15 \\
Voxel-level fMRI
  & 1{,}000 & 3 & 32
  & 32 & 32
  & 0.03 & 500 & 70/15/15 \\
LargeST-GLA-1500
  & 1{,}500 & 3 & 64
  & 64 & 64
  & $10^{-4}$ & 200 & 60/20/20 \\
LargeST-GLA-3000
  & 3{,}000 & 3 & 64
  & 64 & 64
  & $10^{-4}$ & 1{,}000 & 60/20/20 \\
\bottomrule
\end{tabular}
\end{table}

\noindent\textit{Simulated EEG note.}
The main simulated EEG benchmark uses the full 100
source-space ROIs for non-latent causal discovery. For the supplementary latent-view experiment, 4 ROIs are withheld during training and 96 ROIs are provided as observed variables; this setting is reported separately in
Appendix~\ref{app:latent_results}. The 64 EEG sensor channels are generated by a synthetic forward projection for qualitative visualization and robustness analysis, and are not the graph dimension used for the main causal discovery evaluation.

\noindent\textit{Simulated fMRI vs.\ voxel-level fMRI.} The simulated fMRI benchmark in the main results table operates at the ROI level ($d=100$, Schaefer-100 parcellation) and is used to evaluate graph-recovery accuracy. The voxel-level fMRI experiment ($d=1000$) is a scalability benchmark; at this scale we use the structural readout $|W_{j,i}|$ described below.

\noindent\textit{Threshold selection and evaluation.}
The binary adjacency matrix is obtained by sweeping 200 evenly-spaced candidate
thresholds between the minimum and maximum score on the validation split and
selecting the threshold that maximizes the F1 score on the validation set.
The chosen threshold is then applied to the test-set scores without further
adjustment.
For voxel-level fMRI we additionally apply a top-$k$ sparsity constraint
($k = 7142$, matching the ground-truth sparsity) before thresholding.
All metrics (AUROC, AUPRC, F1, SHD, precision, recall) are computed on
off-diagonal entries only; self-edges ($i=j$) are excluded throughout.
AUROC and AUPRC are computed with \texttt{sklearn.metrics.roc\_auc\_score}
and \texttt{average\_precision\_score}, respectively.

\noindent\textit{Compute.} All SCDM runs used one NVIDIA A100 80\,GB PCIe GPU on the \texttt{epyc-gpu} SLURM partition. NC8 jobs used 32\,GB CPU memory and a 3-hour limit; simulated fMRI multi-seed runs used a 1.5-hour limit; voxel-level fMRI runs used 64\,GB CPU memory and a 4-hour limit. The five-method LargeST-GLA-3000 feasibility comparison in Table~\ref{tab:gla3000_feasibility} requested 128\,GB of host memory and an 80-minute limit per run. For this comparison, peak GPU memory denotes the maximum usage sampled by \texttt{nvidia-smi} during each run; for failed runs, it is the maximum recorded before failure. Such polling may miss short-lived allocation peaks. GPU memory was not recorded for the controlled benchmarks.

\noindent\textit{Computational cost.} Let $C_f(d,p)$ be the cost of one forward and backward pass of the predictor on one sample; for the MLP above with hidden widths $h_1=128$ and $h_2=64$, $C_f(d,p)=O(dprh_1+h_1h_2)$. One meta-episode adapts $B$ tasks for $K$ steps on $L_{\mathrm{sup}}$ samples and evaluates $L_{\mathrm{qry}}$ query samples, so it costs $O(B(KL_{\mathrm{sup}}+L_{\mathrm{qry}})C_f(d,p))$, and $N$ episodes cost $O(NB(KL_{\mathrm{sup}}+L_{\mathrm{qry}})C_f(d,p))$. First-order MAML avoids second-order terms through the inner loop. With fixed window lengths, $T$ affects only the number of available windows and data storage, not the cost of an episode. Memory is dominated by $W$ ($O(d^2r)$), the first predictor layer ($O(dprh_1)$), and the temporary adapted copies for the $B$ sampled tasks. The structural readout costs $O(d^2)$; the source-isolated readout costs $O(d^2 M L_{\mathrm{qry}}C_f(d,p))$ plus $dM$ adaptations. These expressions count work per update or per readout; they do not bound the number of updates needed to converge.

\paragraph{Post-training graph readout.}\label{app:score_readout_details}
The readout is applied only after training and is not part of the optimization objective. We use the main-text convention: $W_{j,i}$ gates source $j$ for target $i$, and $s_{j\to i}$ is the score of edge $j\to i$.

\noindent\textit{Source-isolated readout.} For each target $i$, we first adapt a temporary copy of the trained parameters on a fresh support window with the same $K=3$ inner steps used in training, because training also measures query losses after adaptation. On the paired query window we then compute two losses. The baseline loss $L^{(i)}_{\mathrm{zero}}$ sets the whole column $W[:,i]$ to zero, which removes the history of every source. The source-isolated loss $L^{(i)}_{j}$ keeps only $W_{j,i}$ and sets $W_{k,i}=0$ for all $k\neq j$, which keeps only the history of source $j$. Because each $W_{k,i}$ multiplies all lags of source $k$, zeroing $W_{k,i}$ is the same as replacing the history of source $k$ with zeros. The score is
\begin{equation}
s_{j\to i}=\frac{L^{(i)}_{\mathrm{zero}}-L^{(i)}_{j}}{L^{(i)}_{\mathrm{zero}}+\varepsilon},\qquad \varepsilon=10^{-8},
\label{eq:pert_score_main}
\end{equation}
so a larger value means that source $j$ alone removes more of the prediction error of target $i$. Scores are averaged over $M$ independent support/query draws ($M=100$ for NC8 and $M=50$ for simulated fMRI and EEG). We use this readout only for $d\le100$ because its cost grows with $d^2$ (see the computational cost paragraph above).

\noindent\textit{Structural readout.} For voxel-level fMRI and LargeST-GLA we use $s_{j\to i}=|W_{j,i}|$ (the $\ell_2$ norm of the $r$ components when $r>1$), which costs $O(d^2)$ and needs no model evaluation.

\subsection{Dataset Overview and Evaluation Protocols}
\label{app:dataset_overview}

This subsection gives a compact overview of the benchmark suite and evaluation protocol. Detailed simulation construction choices, including graph-generation rules, dynamics, observation models, and recoverability diagnostics, are provided separately in Appendix~\ref{app:simulation_setups}.

\paragraph{Benchmark overview.}
NC8 is a low-dimensional nonlinear temporal benchmark with eight observed variables and a known lagged directed graph. It is used to verify causal recovery in a small controlled setting and to support the masked-NC8 latent diagnostic.
The simulated fMRI benchmark is a TVB-inspired, atlas-grounded BOLD-like simulation over 100 Schaefer cortical ROIs. It combines Yeo-network organization, delayed nonlinear ROI interactions, and HRF-based observation effects while retaining a known directed ground-truth graph.
The simulated EEG benchmark is a source-space EEG-like simulation over the same Schaefer-100 ROI scale. It uses network-modular directed source interactions, EEG-like oscillatory components, and a simplified 64-channel source-to-sensor projection for qualitative visualization.
The voxel-level fMRI benchmark scales the controlled setting to 1000 observed variables and nearly one million candidate directed pairs, preserving synthetic ground truth for runtime--accuracy evaluation.
LargeST-GLA is a real-world traffic benchmark without complete directed causal ground truth; it is therefore used for scalability, forecasting, and qualitative structural analysis rather than causal AUROC/AUPRC.

\paragraph{Evaluation protocol.}
For datasets with known ground-truth graphs, we report AUROC and AUPRC as threshold-independent causal discovery metrics. When threshold-dependent metrics such as F1 and SHD are reported, the threshold is selected on the validation split and applied once to the test split. Unless otherwise stated, self-edges are excluded from graph-level evaluation because they correspond to autoregressive self-dependencies rather than cross-variable causal discovery. For LargeST, we do not report causal AUROC or AUPRC because a complete directed causal graph is not available.

\begin{table}[t]
\centering
\small
\caption{Summary of datasets used in the experiments. 
The symbol $d$ denotes the number of observed variables provided to the learner, and $q$ denotes withheld or latent variables when applicable. 
For LargeST, chronological train/validation/test splits are used and the complete directed causal ground truth is unavailable.}
\label{tab:dataset_summary}
\begin{tabular}{lcccl}
\toprule
\textbf{Dataset} & $\boldsymbol{d}$ & $\boldsymbol{q}$ & \textbf{GT graph} & \textbf{Primary role} \\
\midrule
NC8 & 8 & 0 & Yes & Controlled causal recovery \\
Masked NC8 & 6 & 2 & Yes & Controlled latent diagnostic \\
Simulated fMRI & 100 & 0 & Yes & Neuroimaging causal recovery \\
Latent-view simulated fMRI & 96 & 4 & Yes & Latent-view observed-block recovery \\
Simulated EEG & 100 & 0 & Yes & Source-space causal recovery \\
Latent-view simulated EEG & 96 & 4 & Yes & Latent-view observed-block recovery \\
Voxel-level fMRI & 1000 & 4 & Yes & High-dimensional scalability \\
LargeST-GLA-1500 & 1500 & 0 & No & Real-world scalability and forecasting \\
LargeST-GLA-3000 & 3000 & 0 & No & Real-world scalability and forecasting \\
\bottomrule
\end{tabular}
\end{table}

\begin{table}[t]
\centering
\small
\caption{Dataset split protocol. Synthetic datasets use train, validation, and test partitions for model fitting, threshold selection, and final reporting. LargeST uses chronological splits to avoid leakage from future intervals.}
\label{tab:dataset_splits}
\begin{tabular}{p{0.30\linewidth}p{0.58\linewidth}}
\toprule
\textbf{Dataset group} & \textbf{Split protocol} \\
\midrule
NC8 and masked NC8 
& Training, validation, and test sequences or replicas are separated before model selection and final reporting. \\
\midrule
Simulated fMRI / EEG and latent-view variants 
& 200/50/50 train/validation/test samples are used for controlled recovery experiments. \\
\midrule
Voxel-level fMRI 
& Train, validation, and test intervals are separated; ground-truth graph evaluation is performed on held-out test scores. \\
\midrule
LargeST-GLA 
& Chronological train/validation/test intervals are used to avoid future-to-past leakage. \\
\bottomrule
\end{tabular}
\end{table}

\begin{table}[t]
\centering
\small
\caption{Evaluation role of each dataset group. Controlled synthetic datasets support causal recovery metrics because directed ground-truth graphs are available. LargeST is used for real-world scalability, forecasting, and qualitative structural analysis because complete directed causal ground truth is unavailable.}
\label{tab:evaluation_usage}
\begin{tabular}{p{0.28\linewidth}p{0.31\linewidth}p{0.31\linewidth}}
\toprule
\textbf{Dataset group} & \textbf{Causal recovery metrics} & \textbf{Auxiliary use} \\
\midrule
NC8, simulated fMRI, simulated EEG 
& AUROC, AUPRC, F1, SHD against known directed graphs 
& Runtime and controlled ablations \\
\midrule
Masked NC8 and latent-view neuroimaging benchmarks 
& Block-level recovery diagnostics under controlled masking 
& Latent-extension analysis \\
\midrule
Voxel-level fMRI 
& AUROC/AUPRC against known high-dimensional synthetic graph 
& Runtime--performance tradeoff and scalability analysis \\
\midrule
LargeST-GLA 
& No causal AUROC/AUPRC because complete directed ground truth is unavailable 
& Real-world scalability, forecasting, and qualitative graph visualization \\
\bottomrule
\end{tabular}
\end{table}

Tables~\ref{tab:dataset_summary}--\ref{tab:evaluation_usage} summarize the benchmark suite, split protocol, and evaluation role of each dataset. 
This organization separates controlled causal recovery experiments from high-dimensional scalability and real-world forecasting experiments.

\begin{figure}[h]
\centering
\includegraphics[width=0.98\linewidth]{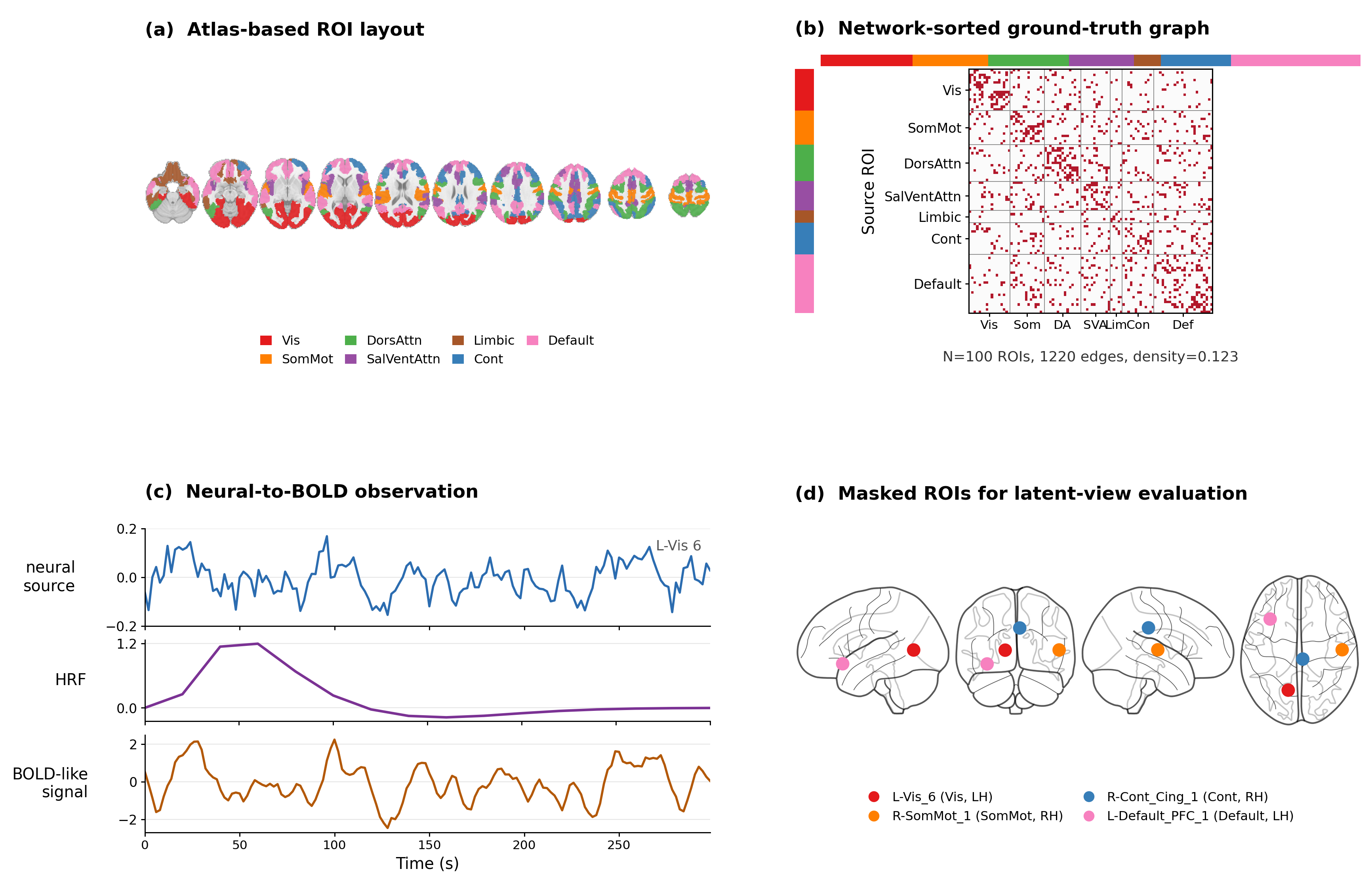}
\caption{
Overview of the simulated fMRI benchmark. 
The benchmark uses Schaefer-100 cortical ROIs grouped by Yeo-7 functional networks. 
The visualization summarizes the atlas-based ROI layout, the network-sorted directed ground-truth graph, representative BOLD-like time series, and the structured lagged effective-connectivity pattern used for evaluation. 
This figure illustrates the construction of the controlled TVB-inspired, atlas-grounded fMRI simulation rather than empirical subject-level fMRI connectivity.
}
\label{fig:simulated_fmri_dataset_overview}
\end{figure}

\paragraph{Neuroimaging simulation overviews.}
Figures~\ref{fig:simulated_fmri_dataset_overview} and~\ref{fig:simulated_eeg_dataset_overview} provide compact visual summaries of the two atlas-grounded neuroimaging-inspired benchmarks. The detailed fMRI and EEG construction procedures are intentionally kept in Appendix~\ref{app:simulated_fmri_setup} and Appendix~\ref{app:simulated_eeg_setup}, respectively, to avoid duplicating the dataset specification here.

\begin{figure}[h]
\centering
\includegraphics[width=0.98\linewidth]{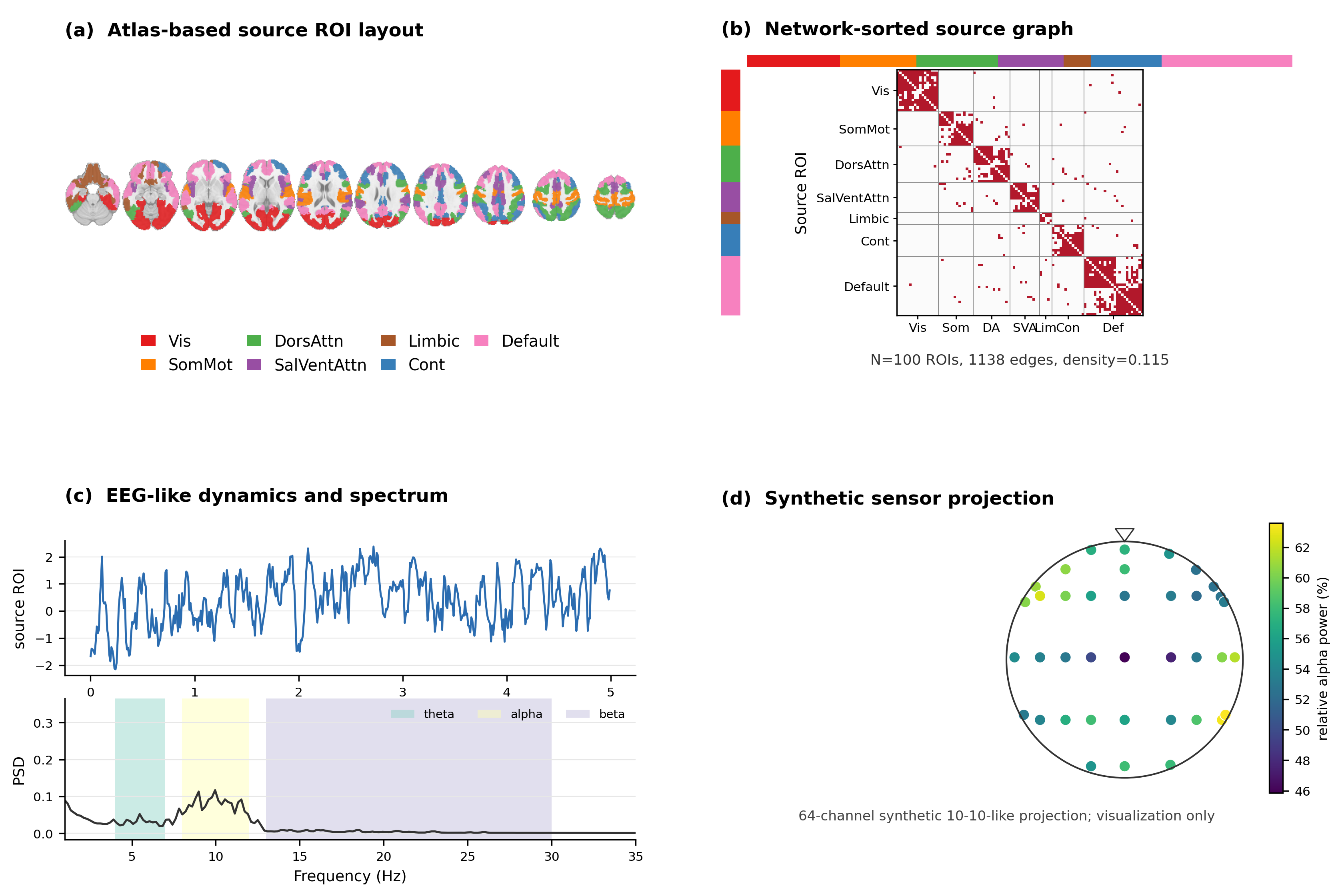}
\caption{
Overview of the simulated EEG benchmark. 
The benchmark is constructed as a Schaefer-100 source-space EEG-like simulation with Yeo-network-modular directed source interactions. 
The visualization summarizes the source ROI organization, the network-sorted source-level ground-truth graph, representative oscillatory source signals, and the synthetic 64-channel source-to-sensor projection used for qualitative visualization. 
The projection is a simplified source-mixing model and is not intended to represent a subject-specific BEM/FEM EEG forward solution.
}
\label{fig:simulated_eeg_dataset_overview}
\end{figure}

\subsection{Additional Experimental Results}
\label{app:supplementary_results}

This subsection reports the high-dimensional voxel-level fMRI runtime--accuracy tradeoff, supplementary LargeST forecasting results, and additional ablations.

\paragraph{High-dimensional voxel-level fMRI results.}\label{app:voxel_fmri_results}

The voxel-level fMRI benchmark is a high-dimensional synthetic benchmark with known ground-truth causal structure. 
It contains $d=1000$ observed variables and $T=800$ time points, yielding $999{,}000$ off-diagonal candidate directed relations after excluding self-edges. 
The ground-truth graph contains $7{,}142$ positive directed off-diagonal edges, corresponding to a positive edge density of $0.7149\%$. 
Since the expected AUPRC of a random ranking equals the positive edge density, raw AUPRC values are numerically small in this setting. 
We therefore additionally report AUPRC enrichment over random:
\[
    \mathrm{AUPRC\ enrichment}
    =
    \frac{\mathrm{AUPRC}}{\mathrm{positive\ edge\ density}}.
\]
All AUROC and AUPRC values are computed from continuous causal scores with self-edges excluded. 
For the fixed-budget runtime experiment, each value is averaged over three random seeds at aligned wall-clock budgets of 200, 400, 600, and 800 seconds.

\paragraph{LargeST-GLA-3000 feasibility.}
\label{app:largest_gla3000}
Table~\ref{tab:gla3000_feasibility} in Section~\ref{sec:largest_scalability} reports the feasibility test on the LargeST-GLA-3000 subset, which contains $3000$ traffic sensors and $8{,}997{,}000$ off-diagonal candidate directed pairs. All five methods use the same split, z-scored preprocessing, one NVIDIA A100 80\,GB GPU per run, three seeds, and an 80-minute limit. A run counts as completed only if it writes a valid directed-score checkpoint for all off-diagonal pairs. SCDM completes training and graph extraction in 56.96\,s on average with 6.9\,GB peak GPU memory. UnCLe writes an intermediate checkpoint after its first training phase and then runs out of memory in its joint phase; for this failed run, the reported peak GPU memory is the peak recorded before the failure. CUTS+, JRNGC, and cMLP start normally but do not write a completed checkpoint before the limit. Because LargeST has no complete causal ground truth, this test measures whether a method can run at this scale, not causal accuracy.
 
\begin{figure}[t]
    \centering
    \includegraphics[width=\linewidth]{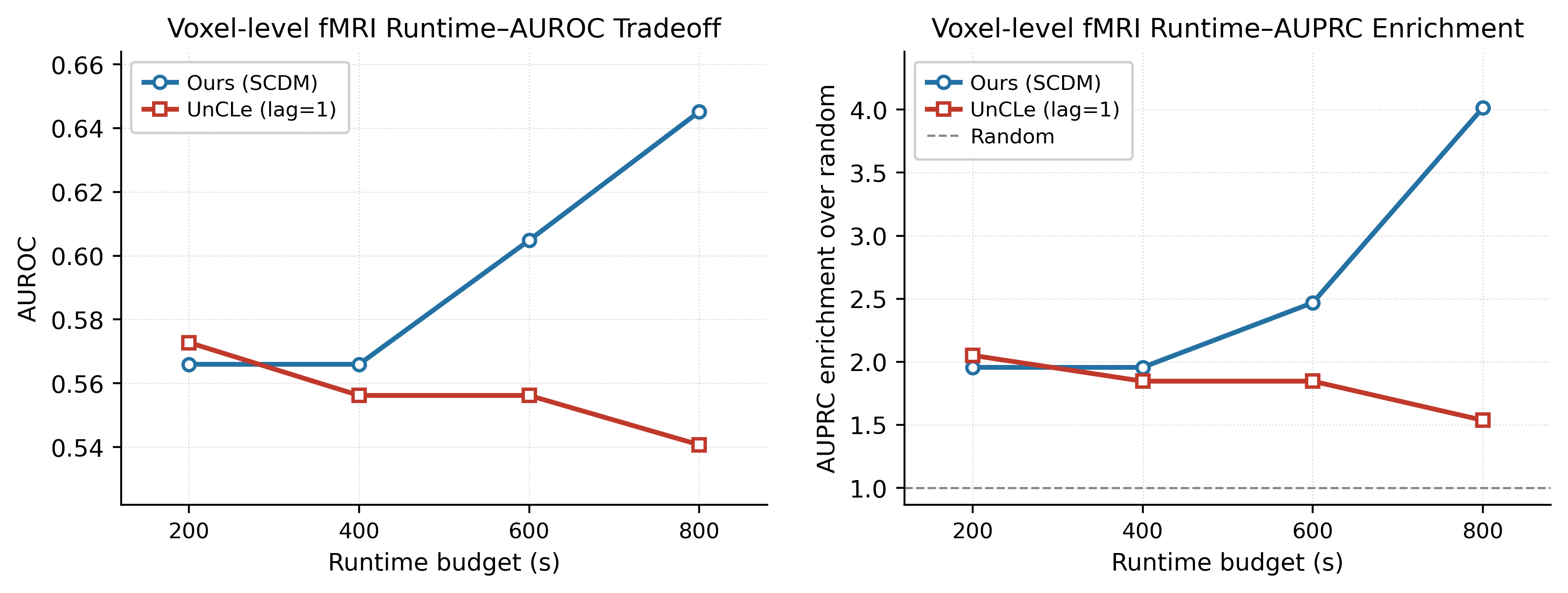}
    \caption{
Runtime--accuracy tradeoff on the high-dimensional voxel-level fMRI benchmark under fixed wall-clock budgets.
Each point reports the mean over three random seeds, without confidence intervals.
Left: AUROC computed from continuous causal scores.
Right: AUPRC enrichment over random, defined as AUPRC normalized by the random-ranking baseline.
UnCLe is slightly ahead at 200 seconds; from 400 seconds on, SCDM has higher AUROC and AUPRC enrichment. 
}
    \label{fig:voxel_fmri_runtime_tradeoff}
\end{figure}
 
Figure~\ref{fig:voxel_fmri_runtime_tradeoff} plots the results of Table~\ref{tab:voxel_fmri_fixed_budget} (Section~\ref{sec:voxel_fmri_scalability}). CUTS+ uses its official configuration with the temporal range set to $L=3$ to match the generating process. Because the voxel-level graph is highly sparse, the enrichment-normalized AUPRC is a more interpretable sparse-edge ranking measure than raw AUPRC. We include UnCLe and CUTS+ in this study because they output continuous directed scores and save intermediate checkpoints, which allows evaluation at matched wall-clock budgets. cMLP and TCDF fit one model per target and PCMCI runs conditional independence tests over variable pairs, lags, and conditioning sets, so they do not fit this fixed-budget protocol, and we did not obtain a completed GVAR run under it; cMLP is instead included in the 3{,}000-variable feasibility test (Appendix~\ref{app:largest_gla3000}).

\paragraph{Supplementary LargeST forecasting diagnostic.}
\label{app:largest_forecasting_results}

Because LargeST does not provide a complete directed ground-truth causal graph, we additionally evaluate the learned structures through a one-step graph-conditioned forecasting diagnostic. We use the same graph-masked RidgeVAR forecaster for all graph priors and vary only the graph used to select candidate parent sensors. Active sensors are selected using train-only standard deviation $\sigma_{\mathrm{train}}\ge 10$, and the same mask is applied to all methods. This experiment is intended to assess whether the learned graph provides useful predictive structure, rather than to benchmark SCDM as a forecasting-specific architecture.

\begin{table}[h]
\centering
\small
\setlength{\tabcolsep}{4pt}
\caption{
Supplementary one-step graph-conditioned forecasting on LargeST-GLA-1500 active sensors.
All variants use the same graph-masked RidgeVAR forecaster with $k=5$, lag $p=3$, and fixed ridge regularization $\alpha=1.0$.
Active sensors are selected using train-only standard deviation $\sigma_{\mathrm{train}}\ge 10$ ($d=1489$).
}
\label{tab:largest_graph_forecasting_aux}
\begin{tabular}{lccc}
\toprule
Graph prior & MAE $\downarrow$ & RMSE $\downarrow$ & $\Delta$MAE vs. Self-only $\uparrow$ \\
\midrule
Self-only AR          & 20.60            & 32.35            & ---   \\
Random                & 19.77 $\pm$ 0.06 & 31.11 $\pm$ 0.16 & +0.83 \\
Distance-kNN          & 20.37            & 31.98            & +0.23 \\
Road-network          & 20.68            & 32.91            & $-$0.08 \\
UnCLe                 & 19.90            & 31.38            & +0.70 \\
Permuted SCDM         & 19.92 $\pm$ 0.33 & 31.44 $\pm$ 0.77 & +0.68 \\
Distance-matched ctrl & 19.72 $\pm$ 0.05 & 31.01 $\pm$ 0.13 & +0.88 \\
\textbf{SCDM}         & \textbf{19.68} $\pm$ \textbf{0.10} & \textbf{30.93} $\pm$ \textbf{0.19} & \textbf{+0.92} \\
\bottomrule
\end{tabular}
\end{table}

Table~\ref{tab:largest_graph_forecasting_aux} shows that the SCDM-derived graph gives the lowest one-step MAE and RMSE among the tested graph priors. The improvement over Self-only AR is modest but consistent, reducing MAE by 0.92. SCDM also improves over the UnCLe-derived graph and the Permuted-SCDM control under the same forecasting backbone. The improvement over Permuted SCDM suggests that the target-specific parent assignment contains predictive signal.

At the same time, the margins over Random and Distance-matched controls are small. We therefore interpret this result as supplementary evidence that the learned graph is forecast-compatible, rather than as a standalone claim of forecasting superiority or causal correctness. Recursive multi-horizon rollouts produced mixed results and are not used as the primary forecasting diagnostic.

\paragraph{Supplementary ablation results.}
\label{app:additional_runtime_metrics}

We provide additional ablations to clarify which components of SCDM contribute to the reported causal recovery performance.
The main text reports the task-batch-size ablation, which isolates the role of target-task subsampling in the shared meta-learning framework.
Here, we further analyze the graph readout step by comparing two scoring rules applied to SCDM outputs: the learned causal-strength score $|W_{ij}|$ and the post-training source-isolated perturbation readout.

\paragraph{Readout ablation.}
This ablation is designed to test whether the reported recovery performance is driven only by the post-training perturbation readout or whether the learned causal-strength parameters already contain meaningful structural information.
For each dataset, we compare the same SCDM training pipeline under two edge-scoring rules.
The first uses the learned absolute causal-strength score $|W_{ij}|$ directly.
The second uses the source-isolated perturbation readout described in Appendix~\ref{app:score_readout_details}.
We do not change the training data, model architecture, optimizer, or evaluation protocol.

\begin{table}[t]
\centering
\small
\caption{Readout ablation on the simulated fMRI and simulated EEG benchmarks. 
Both readouts are applied to SCDM outputs. 
The learned $|W|$ score evaluates whether the trained causal-strength matrix itself contains useful structural information, while the source-isolated perturbation readout evaluates an alternative post-training graph-extraction procedure.
Results are reported as mean $\pm$ 95\% confidence interval.}
\label{tab:readout_ablation}
\begin{tabular}{llcc}
\toprule
\textbf{Dataset} & \textbf{Readout} & \textbf{AUROC $\uparrow$} & \textbf{AUPRC $\uparrow$} \\
\midrule
Simulated fMRI 
& Learned $|W|$ score 
& $0.867 \pm 0.006$ 
& $0.728 \pm 0.019$ \\
Simulated fMRI 
& Source-isolated perturbation readout 
& $0.832 \pm 0.003$ 
& $0.693 \pm 0.010$ \\
\midrule
Simulated EEG 
& Learned $|W|$ score 
& $0.713 \pm 0.053$ 
& $0.426 \pm 0.086$ \\
Simulated EEG 
& Source-isolated perturbation readout 
& $0.717 \pm 0.037$ 
& $0.476 \pm 0.067$ \\
\bottomrule
\end{tabular}
\end{table}

Table~\ref{tab:readout_ablation} compares the structural readout and the source-isolated readout on the same trained models. On simulated fMRI, the structural score $|W|$ has higher AUROC and AUPRC, so useful ranking information is already present in the learned matrix. On simulated EEG, the source-isolated readout has similar AUROC and higher AUPRC. The better readout therefore depends on the dataset, and neither readout implements a causal intervention.

\paragraph{Collinearity stress test.}\label{app:collinear}
The source-isolated score measures predictive contribution, so it can fail when two sources carry nearly the same information. We build NC8-Collinear, which adds a common driver $X_0$ with edges $X_0\to X_1$ and $X_0\to X_2$, a true edge $X_1\to X_3$, and no edge $X_2\to X_3$. $X_2$ is therefore a correlated non-parent that can stand in for $X_1$. We vary the driver strength $a$ over six levels with five seeds each and record the mean absolute correlation between $X_1$ and $X_2$, the condition number of the lagged input, the graph AUPRC, and the number of runs in which the non-edge $X_2\to X_3$ scores above the true edge $X_1\to X_3$ (a proxy reversal). Graph-level ranking stays useful (AUPRC between 0.741 and 0.965), but individual attribution weakens as collinearity grows: the first reversal appears at a mean absolute correlation of 0.572, and reversals occur in 5 of 30 runs overall and in 2 of 5 runs at the highest level. At $a=0$ the edges from $X_0$ carry no signal but remain in the ground truth, which explains the lower AUPRC at that level. Conclusions about individual edges should therefore be treated with care when observed sources are nearly redundant.

\begin{table}[h]
\centering
\small
\caption{NC8-Collinear stress test (five seeds per level). Reversals count runs in which the non-edge $X_2\to X_3$ scores above the true edge $X_1\to X_3$.}
\label{tab:collinear}
\begin{tabular}{@{}ccccc@{}}
\toprule
Driver $a$ & Mean $|\mathrm{corr}(X_1,X_2)|$ & Condition number & AUPRC $\uparrow$ & Reversals $\downarrow$ \\
\midrule
0.0 & 0.037 & 8.45 & 0.741 & 0/5 \\
0.2 & 0.252 & 9.91 & 0.960 & 0/5 \\
0.4 & 0.572 & 12.76 & 0.955 & 1/5 \\
0.6 & 0.750 & 17.09 & 0.947 & 1/5 \\
0.8 & 0.841 & 21.81 & 0.965 & 1/5 \\
1.0 & 0.892 & 26.69 & 0.915 & 2/5 \\
\bottomrule
\end{tabular}
\end{table}

\paragraph{Task-batch-size ablation.}
The task-batch-size ablation in the main text isolates the role of target-task subsampling.
All variants use the same full SCDM model, dataset, optimizer, training budget, and post-training score readout; only the number of target-variable tasks sampled per meta-episode is changed.
Using all target variables at every update removes target-task subsampling, but substantially increases the per-update computational cost and yields weaker recovery under the same training budget.
Moderate task batches provide a better balance between target coverage and optimization efficiency, supporting target-task subsampling as a practical scalability mechanism for high-dimensional temporal causal discovery.

\paragraph{Component ablation on simulated fMRI.}
\label{app:component_ablation_schaefer100}
We further isolate two central design choices of SCDM on the simulated fMRI benchmark: target-task subsampling and the shared structural matrix. All variants use the same dataset, training budget, optimizer, thresholding rule, and evaluation protocol. The full model uses target-task
subsampling with $B=32$ and a trainable shared causal-strength matrix $W$. The no-subsampling variant uses all $100$ target variables at each update. The MLP-only variant removes the trainable shared structural matrix and uses the same evaluation protocol for a fair comparison.

Table~\ref{tab:component_ablation_schaefer100} shows that removing target-task
subsampling substantially degrades graph recovery, reducing AUROC from
$0.867$ to $0.747$ and AUPRC from $0.728$ to $0.552$. Removing the shared
structural matrix also substantially weakens recovery, especially in AUPRC,
which drops from $0.728$ to $0.390$. Under the same training and evaluation protocol, a shared MLP without the structural matrix therefore recovers the graph much less accurately.

\begin{table}[t]
\centering
\caption{Component ablation on the Schaefer100 TVB-inspired simulated fMRI benchmark. All variants use the same training budget and evaluation protocol. AUROC and AUPRC are reported as mean $\pm$ 95\% confidence interval; F1 and SHD are reported as means.}
\label{tab:component_ablation_schaefer100}
\resizebox{\linewidth}{!}{
\begin{tabular}{lcccccc}
\toprule
Variant & Task subsample & Shared $W$ & AUROC $\uparrow$ & AUPRC $\uparrow$ & F1 $\uparrow$ & SHD $\downarrow$ \\
\midrule
Full SCDM
& \checkmark & \checkmark
& $\mathbf{0.867 \pm 0.006}$ & $\mathbf{0.728 \pm 0.019}$ & $\mathbf{0.675}$ & $\mathbf{770.6}$ \\
w/o task subsampling
& -- & \checkmark
& $0.747 \pm 0.008$ & $0.552 \pm 0.015$ & $0.392$ & $1020.8$ \\
w/o shared $W$ (MLP-only)
& \checkmark & --
& $0.732 \pm 0.012$ & $0.390 \pm 0.022$ & $0.401$ & $1486$ \\
\bottomrule
\end{tabular}}
\end{table}

\paragraph{Held-out-target protocol.}\label{app:heldout}
For the support/query diagnostic, the 100 targets of the simulated fMRI benchmark were partitioned into 70 meta-training targets, 10 validation targets, and 20 final-test targets. The reported diagnostic evaluates the 10 validation targets using seed 0; the final-test targets remain untouched. Joint MTL updates the shared parameters directly on the pooled support and query losses of the sampled targets, and the $K=0$ variant keeps the episodic query update but skips inner adaptation. At test time each variant adapts to each held-out target on $S=256$ support samples with $K_{\mathrm{test}}=50$ updates of $\theta$ and $W[:,i]$, and we report the mean squared error on a disjoint query window that is never used for adaptation. This measures how well the learned parameters adapt to unseen targets, not graph recovery.

\begin{table}[h]
\centering
\small
\caption{Held-out-target adaptation. All variants share architecture, data, initialization, training budget, and test-time adaptation; only the training objective differs.}
\label{tab:heldout_meta}
\begin{tabular}{@{}lcc@{}}
\toprule
Training variant & Inner steps in training & Post-adaptation query MSE $\downarrow$ \\
\midrule
Joint MTL (sampled targets) & -- & 4.305 \\
Episodic, no inner adaptation & $K=0$ & 5.487 \\
Full SCDM & $K=3$ & \textbf{3.782} \\
\bottomrule
\end{tabular}
\end{table}

\section{Latent Recovery and Results}
\label{app:latent_results}

This section reports supplementary experiments for the latent-augmented extension of SCDM. 
The main paper focuses on fully observed temporal causal discovery, where all variables involved in the target graph are available to the learner. 
Here, we consider a controlled partially observed setting in which a subset of variables is withheld from the learner input and treated as latent. 
The purpose of these experiments is not to solve unrestricted latent causal discovery, but to evaluate whether the proposed shared-parameter meta-learning framework can be extended to represent hidden-to-observed structural influence under controlled masking.

\subsection{Latent-Augmented Setting and Evaluation Protocol}
\label{app:latent_protocol}

We construct the latent benchmark from a fully observed temporal system by deliberately withholding a subset of variables from the learner input. 
Let
\[
s_t =
\begin{bmatrix}
x_t \\
z_t
\end{bmatrix}
\in \mathbb{R}^{d_{\mathrm{obs}} + q},
\]
where $x_t \in \mathbb{R}^{d_{\mathrm{obs}}}$ denotes the observed variables provided to the learner and $z_t \in \mathbb{R}^{q}$ denotes the withheld variables. 
During training and inference, the model receives only the observed trajectory $x_{1:T}$. 
The withheld trajectories and the full ground-truth graph are used only for evaluation.

The full structural matrix is partitioned as
\[
A =
\begin{bmatrix}
A_{XX} & A_{XZ} \\
A_{ZX} & A_{ZZ}
\end{bmatrix},
\]
where $A_{XX}$ denotes observed-to-observed edges, $A_{ZX}$ denotes latent-to-observed edges, $A_{XZ}$ denotes observed-to-latent edges, and $A_{ZZ}$ denotes latent-to-latent edges. 
In our controlled masking protocol, the withheld variables are selected so that they have no observed parents, yielding
\[
A_{XZ} = 0.
\]
This design avoids a setting in which the withheld variables are merely downstream effects of observed variables, and instead focuses on whether the model can represent hidden drivers of observed dynamics.

For the latent-augmented model, we extend the structural score matrix to the augmented state space:
\[
W =
\begin{bmatrix}
W_{XX} & 0 \\
W_{ZX} & W_{ZZ}
\end{bmatrix}.
\]
Here, $W_{XX}$ captures observed-to-observed temporal influence, $W_{ZX}$ captures latent-to-observed influence, and $W_{ZZ}$ captures latent temporal dependency. 
The zero block follows the controlled benchmark construction, where withheld variables are selected to have no observed parents.

The latent variables are learnable internal states $z_t\in\mathbb{R}^{q}$, not supervised targets: the withheld trajectories are never used for training or scoring. For an observed target $i$, the latent-aware predictor is
\begin{equation}
\widehat{x}^{(i)}_t=f_\theta\!\left(\mathrm{vec}\big(\widetilde H^{X,(i)}_t\big),\;W^{ZX}_{:,i}\odot z_{t-1}\right),\qquad \widetilde H^{X,(i)}_{t,j,\ell}=W^{XX}_{j,i}\,x^{(j)}_{t-\ell},
\end{equation}
so latent states enter each target only through its column of $W^{ZX}$, gated in the same way as observed sources. For a support or query window $b$, the task loss is
\begin{equation}
\mathcal{L}^{b}_{i}=\mathcal{L}^{b}_{\mathrm{pred},i}+\lambda_W\big\|W^{XX}_{:,i}\big\|_1+\lambda_z R^{b}_{z},
\end{equation}
where $\mathcal{L}^{b}_{\mathrm{pred},i}$ is the one-step squared error of observed target $i$ and $R^{b}_{z}$ penalizes the magnitude of the latent states in the window. Without $R^{b}_{z}$, the model could make the latent states arbitrarily large and the weights in $W^{ZX}$ arbitrarily small without changing any prediction.

After training, we extract scores over the augmented state space and evaluate the corresponding graph blocks:
\[
S =
\begin{bmatrix}
S_{XX} & 0 \\
S_{ZX} & S_{ZZ}
\end{bmatrix}.
\]
The $XX$ block is used to evaluate observed-to-observed recovery, while the $ZX$ block is used to evaluate hidden-to-observed structural influence. 
For AUROC and AUPRC, continuous scores are compared directly with the corresponding ground-truth graph blocks. 
For threshold-dependent metrics such as F1 and SHD, the threshold is selected on the validation split and applied once to the test split. 
Because learned latent coordinates are identifiable only up to permutation, metrics involving latent variables are computed after matching learned latent coordinates to the withheld variables.

\subsection{Masked NC8 Latent Recovery}
\label{app:latent_nc8}

We first use masked NC8 as a controlled diagnostic for the latent-augmented model. 
Two variables are withheld from the learner input and treated as latent variables, while the remaining six variables are used as observed inputs. 
Since the full eight-variable ground-truth graph is available, this setting allows direct evaluation of both the augmented full graph and the hidden-to-observed block.

This experiment is not intended as a baseline comparison on the small observed block. 
Instead, it tests whether the latent-augmented model can assign meaningful structural scores to withheld variables under a controlled masking protocol. 
For observed-to-observed recovery, we report both the learned structural score and the perturbation-based readout. 
For hidden-to-observed recovery, we use the learned latent-to-observed structural score because the withheld variables are not observed input channels.

\begin{table}[t]
\centering
\small
\caption{Controlled latent recovery on masked NC8. Upper block: full augmented graph and latent-to-observed (ZX) recovery with structural scores. Lower block: observed-to-observed (XX) recovery with the structural score and with the source-isolated readout. Mean $\pm$ 95\% CI over five graph replicas and three seeds.}
\label{tab:latent_nc8_recovery}
\begin{tabular}{@{}lcccc@{}}
\toprule
\multicolumn{5}{@{}l@{}}{\textit{Aggregate and latent-to-observed recovery using structural scores}} \\
\midrule
\textbf{Evaluation target} & \textbf{AUROC $\uparrow$} & \textbf{AUPRC $\uparrow$} & \textbf{F1 $\uparrow$} & \textbf{SHD $\downarrow$} \\
\midrule
Full augmented graph & $0.926 \pm 0.008$ & $0.634 \pm 0.026$ & $0.598 \pm 0.057$ & $4.87 \pm 0.49$ \\
Latent $\to$ observed (ZX) & $0.606 \pm 0.040$ & $0.680 \pm 0.028$ & $0.648 \pm 0.059$ & $3.80 \pm 0.44$ \\
\midrule
\multicolumn{5}{@{}l@{}}{\textit{Observed-block score diagnostic}} \\
\midrule
\textbf{Evaluation target} & \textbf{Score readout} & \textbf{AUROC $\uparrow$} & & \\
\midrule
Observed $\to$ observed (XX) & Structural score $|W|$ & $0.943 \pm 0.007$ & & \\
Observed $\to$ observed (XX) & Source-isolated readout & $0.761 \pm 0.071$ & & \\
\bottomrule
\end{tabular}
\end{table}

Table~\ref{tab:latent_nc8_recovery} shows that the latent-augmented model achieves strong aggregate recovery on the controlled masked NC8 benchmark. 
The structural score recovers the augmented full graph with high AUROC. 
For the hidden-to-observed block, the model obtains non-trivial recovery after latent-coordinate matching. 
Although $ZX$ AUROC remains modest, the high $ZX$ AUPRC and thresholded metrics indicate that the model can rank sparse hidden-to-observed influences above many non-edges under the controlled masking protocol. 
We therefore treat this experiment as a proof-of-concept diagnostic for latent recovery, while the main scalability claims remain focused on the fully observed setting.

\subsection{Observed-block Recovery under Controlled Masking}
\label{app:latent_observed_block}

We next evaluate the latent-augmented model on the controlled latent-view simulated EEG and simulated fMRI benchmarks. 
Unlike the masked NC8 diagnostic in Appendix~\ref{app:latent_nc8}, which directly examines aggregate and latent-to-observed recovery in a small system, this experiment focuses on whether the latent-augmented model preserves strong observed-to-observed recovery in higher-dimensional neuroimaging-inspired systems.

For both datasets, four variables are withheld from the learner input and treated as latent variables, leaving 96 observed variables for training and observed-block graph recovery. 
The full 100-node ground-truth graph is retained only for evaluation. 
For baselines and our latent-view model, we report recovery on the observed-to-observed XX block. 
We do not use these experiments as the main evidence for unrestricted hidden causal recovery; instead, they test whether controlled latent augmentation remains compatible with accurate observed-block causal discovery.

\begin{table}[t]
\centering
\footnotesize
\setlength{\tabcolsep}{4pt}
\renewcommand{\arraystretch}{0.95}
\caption{
Observed-block causal recovery under controlled latent-view masking on simulated EEG and fMRI.
Results are reported as mean $\pm$ 95\% confidence interval.
Best results are shown in \textbf{bold} and second-best results are \underline{underlined}.
}
\label{tab:latent_observed_block_results}
\begin{tabular}{lcccc}
\toprule
\multirow{2}{*}{Method}
& \multicolumn{2}{c}{Simulated EEG}
& \multicolumn{2}{c}{Simulated fMRI} \\
\cmidrule(lr){2-3}
\cmidrule(lr){4-5}
& XX AUROC $\uparrow$ & XX AUPRC $\uparrow$
& XX AUROC $\uparrow$ & XX AUPRC $\uparrow$ \\
\midrule
UnCLe
& 0.566 $\pm$ 0.005
& 0.146 $\pm$ 0.002
& \underline{0.626 $\pm$ 0.033}
& 0.228 $\pm$ 0.033 \\
PCMCI
& \underline{0.745 $\pm$ 0.011}
& \underline{0.388 $\pm$ 0.019}
& 0.566 $\pm$ 0.015
& 0.156 $\pm$ 0.012 \\
TCDF
& 0.681 $\pm$ 0.006
& 0.213 $\pm$ 0.012
& 0.590 $\pm$ 0.009
& 0.179 $\pm$ 0.011 \\
GVAR
& \textbf{0.747 $\pm$ 0.003}
& 0.339 $\pm$ 0.016
& 0.619 $\pm$ 0.014
& 0.214 $\pm$ 0.014 \\
CUTS+
& 0.589 $\pm$ 0.001
& 0.154 $\pm$ 0.000
& 0.527 $\pm$ 0.012
& 0.137 $\pm$ 0.006 \\
cMLP
& 0.659 $\pm$ 0.007
& 0.256 $\pm$ 0.020
& 0.623 $\pm$ 0.010
& \underline{0.235 $\pm$ 0.016} \\
Ours-Latent
& 0.709 $\pm$ 0.042
& \textbf{0.470 $\pm$ 0.077}
& \textbf{0.849 $\pm$ 0.011}
& \textbf{0.692 $\pm$ 0.013} \\
\bottomrule
\end{tabular}
\end{table}

Table~\ref{tab:latent_observed_block_results} shows that the latent-augmented model remains competitive on observed-block recovery under controlled masking. 
On simulated EEG, the method achieves the best $XX$ AUPRC, suggesting improved sparse-edge ranking even when its AUROC is not the highest. 
On simulated fMRI, the latent-augmented model obtains the best $XX$ AUROC and AUPRC, indicating that the latent extension preserves strong observed-block recovery in this structured neuroimaging-inspired setting.

\subsection{Interpretation and Limitations}
\label{app:latent_limitations}

The latent experiments should be interpreted as a controlled extension study rather than the central contribution of the paper. 
The main SCDM framework targets fully observed scalable temporal causal discovery, and the main empirical claims are based on non-latent causal recovery, high-dimensional scalability, and real-world LargeST evaluation.

The latent-augmented setting introduces explicit hidden nodes, but the learned latent coordinates are not identifiable in an absolute sense. 
They may be permuted, rescaled, or transformed without changing the observed prediction objective. 
Therefore, we do not claim recovery of the exact latent trajectories. 
Instead, we evaluate whether the model can assign useful structural scores to latent-to-observed influence under a controlled masking protocol where the withheld variables and full graph are known for evaluation.

These results provide supplementary evidence that the shared-parameter meta-learning framework can be extended beyond the fully observed setting. 
The masked NC8 experiment provides an interpretable proof of concept for hidden-to-observed recovery, while the masked simulated EEG and fMRI experiments evaluate observed-block recovery in higher-dimensional temporal systems. 
However, unrestricted latent causal discovery under arbitrary hidden confounding, instantaneous effects, or misspecified latent dimension remains outside the scope of this work.

\section{Details of Simulation Datasets}
\label{app:simulation_setups}

This appendix describes the construction of the simulation datasets used in the experiments. 
The purpose of these datasets is to provide controlled temporal systems with known directed ground-truth graphs, while retaining structured domain-inspired properties beyond purely random synthetic benchmarks. 
We include the low-dimensional NC8 benchmark, two neuroimaging-inspired source-level simulations, and a high-dimensional voxel-level fMRI benchmark. 
Dataset usage, train/validation/test splits, evaluation protocols, and compact overview figures are summarized separately in Appendix~\ref{app:dataset_overview}. The present appendix focuses on construction details and omits repeated overview material.

\subsection{NC8 Benchmark}
\label{app:nc8_setup}

NC8 is used as a low-dimensional controlled benchmark for temporal causal discovery. 
It contains eight variables generated from a nonlinear temporal system with a known directed lagged causal graph. 
The benchmark is useful for verifying whether a method can recover directed temporal dependencies in a setting where the graph is small enough for direct inspection and thresholded graph comparison.

For the fully observed setting, all eight variables are provided to the learner and the learned directed score matrix is evaluated against the known eight-node ground-truth graph. 
We report AUROC and AUPRC as the primary threshold-independent metrics. 
When binary graph metrics are reported, the threshold is selected on the validation split and applied once to the test split.

For the controlled latent-view setting, two variables are withheld from the learner input and treated as latent variables. 
The learner receives only the remaining six observed variables, while the full eight-variable graph is retained for evaluation. 
This setting is used only for the latent-augmented extension in Appendix~\ref{app:latent_results}, where we evaluate observed-to-observed and latent-to-observed structural recovery.

\begin{table}[t]
\centering
\small
\caption{Statistics of the NC8 benchmark and its controlled latent-view variant.}
\label{tab:nc8_setup}
\begin{tabular}{lcc}
\toprule
\textbf{Statistic} & \textbf{NC8} & \textbf{Masked NC8} \\
\midrule
Total variables in full system & 8 & 8 \\
Observed variables provided to learner & 8 & 6 \\
Withheld / latent variables & 0 & 2 \\
Time points & 2000 & 2000 \\
Ground-truth graph available & Yes & Yes \\
Maximum lag & 16 & 16 \\
Primary metrics & AUROC, AUPRC, F1, SHD & $XX$ / $Z{\to}X$ AUROC, AUPRC \\
Primary use & Low-dimensional recovery & Latent-extension evaluation \\
\bottomrule
\end{tabular}
\end{table}

\subsection{Simulated fMRI Benchmark}
\label{app:simulated_fmri_setup}

The simulated fMRI benchmark is a TVB-inspired, atlas-grounded, BOLD-like temporal simulation with known directed effective connectivity. 
It is designed to evaluate causal discovery under neuroimaging-inspired spatial organization, delayed neural interactions, and hemodynamic measurement distortion, while preserving full access to the ground-truth graph.

Each variable corresponds to a cortical region of interest (ROI) from the Schaefer-100 parcellation. 
ROIs are associated with Yeo-7 functional-network labels, hemisphere labels, and MNI-space centroids. 
These anatomical and functional metadata are used to construct a structured directed graph and to visualize recovered interactions in brain space.

\begin{table}[t]
\centering
\small
\caption{Atlas and anatomical metadata used in the simulated fMRI benchmark.}
\label{tab:fmri_atlas_metadata}
\begin{tabular}{ll}
\toprule
\textbf{Component} & \textbf{Setting} \\
\midrule
Parcellation & Schaefer-100 cortical atlas \\
Number of ROIs & 100 \\
Coordinate space & MNI152 \\
Functional labels & Yeo-7 canonical networks \\
Hemisphere labels & Left/right hemisphere labels retained \\
ROI metadata & ROI name, hemisphere, Yeo network, and MNI centroid \\
Use in simulation & Distance-aware and network-aware graph construction \\
Use in visualization & Brain-template and ROI-coordinate visualization \\
\bottomrule
\end{tabular}
\end{table}

The ground-truth graph is a sparse directed lagged effective-connectivity graph over the 100 ROIs. 
Edge probabilities are informed by inter-ROI distance, hemisphere membership, and functional-network organization. 
Nearby ROIs are assigned higher prior probability, but long-range, cross-network, and cross-hemisphere interactions are retained to avoid a purely local graph. 
This construction should be interpreted as an atlas-informed structural prior rather than empirical diffusion-MRI, DTI, or ENIGMA structural connectivity.

\begin{table}[t]
\centering
\small
\caption{Ground-truth graph construction in the simulated fMRI benchmark.}
\label{tab:fmri_graph_construction}
\begin{tabular}{ll}
\toprule
\textbf{Component} & \textbf{Implementation} \\
\midrule
Node definition & Cortical ROIs from the Schaefer-100 atlas \\
Topology type & Sparse directed lagged effective-connectivity graph \\
Structural prior & Distance-, hemisphere-, and network-informed edge sampling \\
Spatial constraint & Nearby ROIs have higher prior probability while long-range edges are retained \\
Network constraint & Within-network and selected cross-network interactions are allowed \\
Hemisphere constraint & Both intra-hemispheric and inter-hemispheric edges are retained \\
Lag structure & Directed effects occur over lags 1--3 \\
Binary adjacency & \texttt{A.npy}, shape $100 \times 100$ \\
Lag-resolved adjacency & \texttt{A\_lags.npy}, shape $3 \times 100 \times 100$ \\
Adjacency convention & $A_{ij}=1$ denotes a directed edge from source $i$ to target $j$ \\
Graph density & 12.32\% off-diagonal density \\
Important distinction & Atlas-informed prior, not empirical ENIGMA/DTI structural connectivity \\
\bottomrule
\end{tabular}
\end{table}

Latent neural source activity is generated using stable nonlinear lagged dynamics. 
Coupling weights are rescaled to satisfy a spectral-radius stability target, preventing unstable trajectories while preserving directed delayed interactions. 
The resulting latent source signals provide the pre-hemodynamic neural activity from which BOLD-like observations are generated.

\begin{table}[t]
\centering
\small
\caption{Neural source dynamics used in the simulated fMRI benchmark.}
\label{tab:fmri_neural_dynamics}
\begin{tabular}{ll}
\toprule
\textbf{Parameter} & \textbf{Value} \\
\midrule
Dynamics type & Stable nonlinear lagged ROI dynamics \\
Self autoregression & 0.50 \\
Coupling scale & 0.25 before stability rescaling \\
Spectral-radius target & 0.87 \\
Process noise standard deviation & 0.045 \\
Slow drive standard deviation & 0.015 \\
Pink-noise strength & 0.015 \\
Lag order & 1--3 \\
Time points per sample & 1000 \\
Final generator seed & 20260503 \\
\bottomrule
\end{tabular}
\end{table}

To model fMRI measurement distortion, latent neural source activity is transformed into BOLD-like observations through ROI-wise hemodynamic response function (HRF) convolution. 
The HRF parameters vary across ROIs, producing temporal smoothing and low-frequency dynamics typical of fMRI measurements. 
We use a double-gamma HRF and add observation noise before per-ROI normalization.

\begin{table}[t]
\centering
\small
\caption{Hemodynamic observation model for the simulated fMRI benchmark.}
\label{tab:fmri_hemodynamic_model}
\begin{tabular}{ll}
\toprule
\textbf{Parameter} & \textbf{Value} \\
\midrule
Observation type & BOLD-like ROI time series \\
TR & 2.0 s \\
HRF type & Double-gamma \\
HRF peak delay & Uniform(4.5, 6.5) s \\
HRF undershoot delay & Uniform(12.0, 18.0) s \\
HRF undershoot scale & Uniform(0.15, 0.30) \\
BOLD noise standard deviation & 0.020 \\
Output normalization & Per-ROI z-scored BOLD-like time series \\
\bottomrule
\end{tabular}
\end{table}

The same generated system supports two evaluation settings. 
In the full-view setting, all 100 ROIs are observed and the dataset is used for the main non-latent causal discovery experiments. 
In the latent-view setting, four anatomically defined ROIs are withheld from the learner input and treated as hidden variables, leaving 96 observed variables and 4 latent variables under the same full ground-truth graph. 
This controlled masking protocol allows us to evaluate observed-to-observed and hidden-to-observed structural recovery in the latent-augmented extension.

\begin{table}[t]
\centering
\small
\caption{Full-view and latent-view versions of the simulated fMRI benchmark.}
\label{tab:fmri_full_latent_view}
\begin{tabular}{lcc}
\toprule
\textbf{Statistic} & \textbf{Full-view} & \textbf{Latent-view} \\
\midrule
Total variables in full system & 100 & 100 \\
Observed variables provided to learner & 100 & 96 \\
Withheld / latent variables & 0 & 4 \\
Time points per sample & 1000 & 1000 \\
TR & 2.0 s & 2.0 s \\
Ground-truth graph available & Yes & Yes \\
Graph type & Directed lagged EC & Directed lagged EC \\
Maximum lag & 3 & 3 \\
Graph density & 12.32\% & 12.32\% \\
Train / validation / test split & 200 / 50 / 50 & 200 / 50 / 50 \\
Primary metrics & AUROC, AUPRC, F1, SHD & $XX$ / $Z{\to}X$ AUROC, AUPRC \\
\bottomrule
\end{tabular}
\end{table}

For the latent-view setting, the full 100-node graph is decomposed into observed and hidden blocks. 
Let $X$ denote the 96 observed ROIs and $Z$ denote the 4 masked ROIs. 
We store the corresponding ground-truth blocks to support block-wise evaluation of the latent-augmented model.

Before model evaluation, we checked simple diagnostic baselines to ensure that the benchmark is neither trivially easy nor unrecoverable. 
Random prediction matches the graph-density baseline, lagged correlation detects partial temporal signal, and Ridge-VAR remains close to random under nonlinear dynamics and hemodynamic filtering. 
Our method substantially improves over these diagnostics.

\begin{table}[t]
\centering
\small
\caption{Recoverability diagnostics for the simulated fMRI benchmark.}
\label{tab:fmri_recoverability}
\begin{tabularx}{\linewidth}{@{}llX@{}}
\toprule
\textbf{Method} & \textbf{AUROC / AUPRC} & \textbf{Interpretation} \\
\midrule
Random baseline & 0.500 / 0.123 & Expected from graph density \\
Lagged correlation max-lag & 0.673 / 0.285 & Simple temporal baseline has signal but remains limited \\
Ridge-VAR & 0.526 / 0.133 & Linear VAR remains near random \\
Ours, 5000 steps & 0.823 / 0.669 & Strong improvement over simple baselines \\
\bottomrule
\end{tabularx}
\end{table}

\paragraph{Scope and limitations.}
The simulated fMRI benchmark is TVB-inspired and atlas-grounded, but it is not a full biophysical TVB simulator. 
In particular, we do not claim to use empirical ENIGMA/DTI structural connectivity, Wilson--Cowan E/I neural mass dynamics, or millisecond-level axonal conduction delays. 
The purpose of the benchmark is to provide a controlled BOLD-like causal discovery setting with known directed ground truth, interpretable ROI variables, hemodynamic temporal distortion, and a controlled latent-view variant.

\subsection{Simulated EEG Benchmark}
\label{app:simulated_eeg_setup}

The simulated EEG benchmark is an atlas-grounded source-space EEG-like simulation with known directed ground truth. 
Each variable corresponds to a cortical source ROI from the Schaefer-100 parcellation, and the directed graph follows a network-modular prior based on the Yeo-7 functional systems. 
Source activity is generated with lagged temporal interactions, broadband $1/f$-like background activity, and theta-, alpha-, and beta-band oscillatory components. 
We additionally provide a simplified 64-channel synthetic source-to-sensor projection to mimic spatial mixing and volume-conduction-like leakage for visualization and robustness analysis.

The benchmark should be interpreted as a controlled source-space EEG-like simulation, not as a subject-specific biophysical EEG forward model. 
In particular, the synthetic projection is not intended to represent a realistic head model, BEM/FEM conductivity model, or empirical leadfield. 
Its purpose is to introduce interpretable source-to-sensor mixing while preserving full access to the source-level ground-truth causal graph for quantitative evaluation.

\begin{table}[!htbp]
\centering
\small
\caption{Atlas and source-space metadata used in the simulated EEG benchmark.}
\label{tab:eeg_source_metadata}
\begin{tabular}{ll}
\toprule
\textbf{Component} & \textbf{Setting} \\
\midrule
Parcellation & Schaefer-100 cortical atlas \\
Number of source ROIs & 100 \\
Functional labels & Yeo-7 canonical networks \\
Hemisphere labels & Left/right hemisphere labels retained \\
ROI metadata & ROI name, hemisphere, Yeo network, and atlas coordinate information \\
Use in simulation & Network-aware source graph construction \\
Use in visualization & Atlas slice visualization and network-sorted adjacency display \\
Important distinction & Source-space ROI simulation, not sensor-space empirical EEG recording \\
\bottomrule
\end{tabular}
\end{table}

The ground-truth structure is a sparse directed lagged graph over the 100 source ROIs. 
Edge probabilities are network-modular: within-network interactions are sampled more densely, while cross-network interactions are retained more sparsely to preserve long-range inter-system coupling. 
This produces a graph that is structured but nontrivial, with dense local blocks and sparse off-block directed edges.

The binary adjacency matrix follows a source-to-target convention: $A[i,j]=1$ denotes a directed edge from source ROI $i$ to target ROI $j$. 
The lag-resolved adjacency $A_{\mathrm{lags}}[\ell,i,j]$ follows the same source-to-target convention at lag $\ell$. 
The stored lagged weight tensor (file \texttt{W\_lags\_target\_source.npy}) is indexed as $[\ell, j, i]$, that is, lag, target, and source, whereas the adjacency files use the source-to-target order $[i, j]$.

\begin{table}[!htbp]
\centering
\small
\caption{Ground-truth source graph construction in the simulated EEG benchmark.}
\label{tab:eeg_graph_construction}
\begin{tabularx}{\linewidth}{@{}lX@{}}
\toprule
\textbf{Component} & \textbf{Implementation} \\
\midrule
Node definition & Cortical source ROIs from the Schaefer-100 atlas \\
Topology type & Sparse directed lagged source effective-connectivity graph \\
Structural prior & Yeo-network-modular edge sampling \\
Within-network structure & Denser directed interactions within the same functional network \\
Cross-network structure & Sparser directed interactions across functional networks \\
Lag structure & Directed effects occur over lags 1--5 \\
Binary adjacency & \texttt{A.npy}, shape $100 \times 100$ \\
Lag-resolved adjacency & \texttt{A\_lags.npy}, shape $5 \times 100 \times 100$ \\
Weight tensor & \texttt{W\_lags\_target\_source.npy}, shape $5 \times 100 \times 100$ \\
Adjacency convention & $A[i,j]=1$ denotes a directed edge $i \to j$ \\
Weight convention & \texttt{W\_lags\_target\_source[$\ell$, j, i]} stores target-by-source weights \\
Number of directed edges & 1138 \\
Graph density & 11.495\% \\
Important distinction & Network-informed synthetic source graph, not empirical anatomical connectivity \\
\bottomrule
\end{tabularx}
\end{table}

Source activity is generated with lagged directed interactions over the ground-truth source graph. 
To make the temporal signals more EEG-like, the generator includes broadband $1/f$-like background activity, mild nonstationary amplitude variation, and network-dependent oscillatory components in theta, alpha, and beta frequency ranges. 
The alpha band is dominant but not degenerate, while theta and beta activity remain present. 
This design avoids a purely periodic oscillator while preserving recoverable lagged temporal dependencies.

\begin{table}[!htbp]
\centering
\small
\caption{Source dynamics and spectral diagnostics for the simulated EEG benchmark.}
\label{tab:eeg_source_dynamics}
\begin{tabular}{ll}
\toprule
\textbf{Parameter} & \textbf{Value} \\
\midrule
Dynamics type & Stable lagged source dynamics with oscillatory components \\
Lag order & 1--5 \\
Time points per sample & 2000 \\
Frequency structure & Theta, alpha, and beta-band source activity \\
Broadband component & $1/f$-like background activity \\
Nonstationarity & Slow amplitude modulation and ROI-level variability \\
Theta power & 0.0887 \\
Alpha power & 0.3197 \\
Beta power & 0.0689 \\
Alpha dominance ratio & 2.0287 \\
Estimated $1/f$ spectral slope & $-1.4742$ \\
Output normalization & Source ROI time series are normalized before model evaluation \\
\bottomrule
\end{tabular}
\end{table}

In addition to the source-level time series used for causal recovery, we generate a 64-channel synthetic sensor projection. 
The projection uses a simplified source-to-sensor mixing matrix to mimic spatial leakage and volume-conduction-like mixing. 
This sensor projection is included for visualization and robustness analysis, but the primary causal discovery benchmark remains the source-level graph over atlas-defined ROIs.

The full sensor signal \texttt{EEG.npy} is generated from all 100 source ROIs. 
Therefore, in the latent-view setting, the withheld source activity can still contribute indirectly to the full synthetic sensor measurements. 
To separate this from the controlled source-level latent experiment, we also store \texttt{EEG\_obs96.npy}, which is generated using only the 96 observed source ROIs.

\begin{table}[!htbp]
\centering
\small
\caption{Synthetic source-to-sensor projection used in the simulated EEG benchmark.}
\label{tab:eeg_sensor_projection}
\begin{tabularx}{\linewidth}{@{}lX@{}}
\toprule
\textbf{Component} & \textbf{Setting} \\
\midrule
Number of synthetic sensors & 64 \\
Sensor layout & Synthetic 10--10-like layout \\
Leadfield / mixing matrix & \texttt{leadfield\_synthetic.npy}, shape $64 \times 100$ \\
Sensor positions & \texttt{sensor\_positions.npy}, shape $64 \times 3$ \\
Full sensor signal & \texttt{EEG.npy}, generated from all 100 source ROIs \\
Observed-only sensor signal & \texttt{EEG\_obs96.npy}, generated from 96 observed source ROIs \\
Raw sensor signal & \texttt{EEG\_raw.npy}, before final normalization or postprocessing \\
Visualization quantity & Relative alpha-band power over synthetic sensors \\
Primary purpose & Spatial mixing visualization and sensor-level qualitative analysis \\
Important distinction & Simplified synthetic projection, not a subject-specific BEM/FEM EEG forward model \\
\bottomrule
\end{tabularx}
\end{table}

The same generated system supports two evaluation settings. 
In the full-view setting, all 100 source ROIs are observed and the dataset is used for the main non-latent causal discovery experiments. 
In the latent-view setting, four source ROIs are withheld from the learner input and treated as hidden variables, leaving 96 observed source ROIs and 4 latent source variables under the same full ground-truth graph. 
The withheld ROIs are selected from different functional systems and hemispheric locations to create a controlled partially observed source-space setting. 
The full graph remains available only for evaluation, while the learner receives only the observed source trajectories.

\begin{table}[!htbp]
\centering
\small
\caption{Full-view and latent-view versions of the simulated EEG benchmark.}
\label{tab:eeg_full_latent_view}
\begin{tabular}{lcc}
\toprule
\textbf{Statistic} & \textbf{Full-view} & \textbf{Latent-view} \\
\midrule
Total source ROIs in full system & 100 & 100 \\
Observed variables provided to learner & 100 & 96 \\
Withheld / latent variables & 0 & 4 \\
Time points per sample & 2000 & 2000 \\
Ground-truth graph available & Yes & Yes \\
Graph type & Directed lagged source EC & Directed lagged source EC \\
Maximum lag & 5 & 5 \\
Graph density & 11.495\% & 11.495\% \\
Train / validation / test split & 200 / 50 / 50 & 200 / 50 / 50 \\
Primary input file & \texttt{X.npy} & \texttt{X\_obs.npy} \\
Hidden source file & N/A & \texttt{Z.npy} \\
Primary metrics & AUROC, AUPRC, F1, SHD & $XX$ / $Z{\to}X$ AUROC, AUPRC \\
\bottomrule
\end{tabular}
\end{table}

For the latent-view setting, the full 100-node graph is decomposed into observed and hidden blocks. 
Let $X$ denote the 96 observed source ROIs and $Z$ denote the 4 withheld source ROIs. 
We store the corresponding graph blocks to support block-wise evaluation of the latent-augmented model.

\begin{table}[h]
\centering
\small
\caption{Withheld ROIs in the latent-view simulated EEG setting.}
\label{tab:eeg_withheld_rois}
\begin{tabular}{ll}
\toprule
\textbf{Index} & \textbf{ROI name} \\
\midrule
5 & \texttt{7Networks\_LH\_Vis\_6} \\
58 & \texttt{7Networks\_RH\_SomMot\_1} \\
86 & \texttt{7Networks\_RH\_Cont\_Cing\_1} \\
41 & \texttt{7Networks\_LH\_Default\_PFC\_1} \\
\bottomrule
\end{tabular}
\end{table}

\begin{table}[!htbp]
\centering
\small
\caption{Ground-truth graph blocks in the latent-view simulated EEG setting.}
\label{tab:eeg_graph_blocks}
\begin{tabular}{ll}
\toprule
\textbf{Block} & \textbf{Description} \\
\midrule
\texttt{A\_xx.npy} & Directed edges among observed source ROIs, $X \to X$ \\
\texttt{A\_xz.npy} & Directed edges from observed source ROIs to hidden source ROIs, $X \to Z$ \\
\texttt{A\_zx.npy} & Directed edges from hidden source ROIs to observed source ROIs, $Z \to X$ \\
\texttt{A\_zz.npy} & Directed edges among hidden source ROIs, $Z \to Z$ \\
\bottomrule
\end{tabular}
\end{table}
\FloatBarrier

Table~\ref{tab:eeg_recoverability} presents preliminary recoverability diagnostics for the simulated EEG benchmark. The main multi-seed causal-recovery results and graph-score readout comparisons are reported in Table~\ref{tab:low_medium_causal_recovery} and Table~\ref{tab:readout_ablation}, respectively.

\begin{table}[H]
\centering
\small
\caption{Recoverability diagnostics for the simulated EEG benchmark. 
The ``Ours'' row is a one-seed smoke diagnostic on the full 100-source non-latent setting and is not a replacement for the final multi-seed experimental table.}
\label{tab:eeg_recoverability}
\begin{tabular}{lll}
\toprule
\textbf{Method} & \textbf{AUROC / AUPRC} & \textbf{Interpretation} \\
\midrule
Random baseline & 0.500 / 0.1149 & Expected from graph density \\
Lagged correlation & 0.6292 / 0.1826 & Simple temporal association has partial signal \\
Ridge-VAR & 0.6983 / 0.2668 & Linear VAR recovers part of the lagged structure \\
Ours, 3000 steps & 0.6933 / 0.4598 & Stronger sparse-edge ranking in AUPRC \\
\bottomrule
\end{tabular}
\end{table}

\paragraph{Scope and limitations.}
The simulated EEG benchmark is atlas-grounded and EEG-like, but it is not intended to be a fully realistic biophysical EEG simulator. 
We do not claim to model individual head geometry, tissue conductivity, skull anisotropy, BEM/FEM forward solutions, or empirical electrode recordings. 
The simplified source-to-sensor projection is used only to mimic source mixing and to provide sensor-space visualizations. 
The purpose of the benchmark is to provide a controlled source-level temporal causal discovery setting with known directed ground truth, interpretable cortical ROI variables, oscillatory EEG-like spectral properties, and a controlled latent-view variant.

\subsection{High-dimensional Voxel-level fMRI Benchmark}
\label{app:voxel_fmri_setup}

The voxel-level fMRI benchmark is designed to evaluate causal discovery in a substantially larger controlled system than the ROI-level simulated fMRI benchmark. 
While the ROI-level benchmark contains $100$ Schaefer cortical ROIs, the voxel-level benchmark contains $d=1000$ observed variables arranged on a $10\times 10\times 10$ spatial grid. 
This yields $d(d-1)=999{,}000$ candidate directed off-diagonal relations, allowing us to stress-test scalability in a regime where exhaustive target-wise modeling becomes expensive. 
Unlike LargeST, the voxel-level fMRI benchmark includes a complete synthetic directed ground-truth graph, so it supports AUROC and AUPRC evaluation in addition to runtime and enrichment analysis.

The benchmark should be interpreted as a controlled high-dimensional BOLD-like temporal simulation rather than a full biophysical simulator of empirical voxel-level fMRI acquisition. 
The generated variables are organized into spatial regions, with local lagged dependencies between neighboring voxels and additional long-range region-to-region influences. 
The simulation also includes latent regional drivers and nonlinear lagged dynamics, followed by fMRI-like observation effects such as hemodynamic response kernels, downsampling, and observation noise. 
Stability is controlled through global inflow scaling rather than an explicitly named spectral-radius rescaling step. 
The purpose of this benchmark is therefore to preserve a known directed temporal graph while increasing the dimensionality to a voxel-scale setting.

For the high-dimensional results reported in this paper, we use the \texttt{voxel1000\_v2} version of the benchmark. 
This version contains $T=800$ time points and a sparse directed observed graph with $7142$ true off-diagonal edges out of $999{,}000$ possible directed pairs, corresponding to a density of approximately $0.715\%$. 
The lag-resolved graph has maximum lag $p=3$ and is collapsed into a directed observed adjacency matrix for graph-level evaluation. 
Self-edges are excluded throughout. 
Because applying source-isolated perturbation readout to all ordered pairs is computationally expensive at this scale, the voxel-level fMRI experiments use the learned edge-strength score $|W_{ij}|$ directly for ranking candidate edges. Table~\ref{tab:voxel_fmri_summary} summarizes the benchmark configuration and evaluation settings.

\begin{table}[H]
\centering
\small
\caption{At-a-glance summary of the high-dimensional voxel-level fMRI benchmark used in the scalability experiments.}
\label{tab:voxel_fmri_summary}
\begin{tabular}{ll}
\toprule
\textbf{Statistic} & \textbf{Value} \\
\midrule
Dataset version & \texttt{voxel1000\_v2} \\
Dataset type & Controlled high-dimensional BOLD-like temporal simulation \\
Observed variables & $1000$ \\
Latent drivers & $4$ \\
Time points & $800$ \\
Spatial layout & $10\times 10\times 10$ voxel grid \\
Lag order & $3$ \\
Candidate directed off-diagonal pairs & $999{,}000$ \\
True directed observed edges & $7142$ \\
Graph density & $0.007149$ \\
Evaluation graph & Lagged edges collapsed into directed observed adjacency \\
Self-edges & Excluded \\
Score used for scalability evaluation & Learned $|W_{ij}|$ score \\
Top-$k$ sparsity matching & $k=7142$ \\
Primary metrics & AUROC, AUPRC, AUPRC enrichment, runtime \\
Primary use & High-dimensional scalability and runtime--performance evaluation \\
\bottomrule
\end{tabular}
\end{table}

This benchmark complements both the ROI-level simulated fMRI benchmark and the real-world LargeST experiments. 
Compared with the ROI-level benchmark, it increases the number of variables by an order of magnitude while retaining a known synthetic directed graph for quantitative recovery evaluation. 
Compared with LargeST, it is not intended to model an empirical traffic or neuroimaging system, but it provides complete ground truth and therefore allows AUROC, AUPRC, and enrichment-based analysis under controlled high-dimensional conditions.

\subsection{Recoverability and Evaluation Scope}
\label{app:simulation_recoverability_scope}

The simulation datasets are intended to support controlled causal recovery evaluation with known directed ground truth. 
They should not be interpreted as fully realistic biomedical simulators. 
Instead, they provide structured temporal systems that combine recoverable directed dependencies, domain-inspired organization, and controlled access to ground-truth graphs.

The low-dimensional NC8 benchmark validates whether a method can recover directed temporal dependencies in a small nonlinear system. 
The simulated fMRI and simulated EEG benchmarks evaluate recovery in medium-dimensional neuroimaging-inspired systems with atlas-based organization and known directed graphs. 
The voxel-level fMRI benchmark increases the number of variables to stress-test runtime, memory, and high-dimensional score extraction. 
Together, these simulations separate causal recovery evaluation, scalability analysis, and real-world qualitative assessment, while avoiding unsupported claims about exact biological realism.

\section{Broader Impacts}
\label{app:broader_impacts}

SCDM is a scalable causal discovery framework for nonlinear temporal systems. By reducing the cost of learning directed dependency scores in high-dimensional time series, the method may support more efficient scientific analysis and decision-making in domains such as neuroscience, traffic systems, finance, climate, and engineering monitoring. In particular, scalable temporal causal discovery can help researchers explore structured dynamic relations in settings where exhaustive target-wise modeling or repeated conditional testing is computationally expensive.

However, the outputs of SCDM should be interpreted with caution. The method is based on observational temporal data and relies on assumptions such as delayed causal influence, sufficient observation of relevant variables, correct lag specification, and informative predictive risk gaps. If these assumptions are violated, learned directed scores may reflect predictive association rather than true causal mechanisms. Misuse of such scores in high-stakes settings, such as clinical decision-making, public policy, infrastructure control, or financial intervention, could lead to misleading conclusions or harmful actions. Therefore, SCDM should be used as a tool for hypothesis generation and structural analysis, and its findings should be validated with domain knowledge, controlled interventions, or external evidence whenever possible.

\FloatBarrier

\newpage

\clearpage
\section*{NeurIPS Paper Checklist}



\begin{enumerate}

\item {\bf Claims}
    \item[] Question: Do the main claims made in the abstract and introduction accurately reflect the paper's contributions and scope?
    \item[] Answer: \answerYes{} 
    \item[] Justification: The abstract and introduction state the main scope and contributions of SCDM, including shared-parameter meta-learning, target-task subsampling, post-training graph readout, and the fully observed delayed-system recovery setting. The corresponding theoretical and experimental support is provided in Sections~\ref{sec:theory}--\ref{sec:experiments} and Appendices~\ref{app:theory}--\ref{app:experiments}.

\item {\bf Limitations}
    \item[] Question: Does the paper discuss the limitations of the work performed by the authors?
    \item[] Answer: \answerYes{} 
    \item[] Justification: Section 4.3 and Section 6 state the assumptions of the recovery result (causal sufficiency, correct lag order, no instantaneous effects, positive risk gap) and explain how a hidden common cause breaks the margin condition. Section 5.2 states that SCDM is not uniformly best on small benchmarks. Appendix B.3 reports a collinearity stress test in which individual edge scores can move to correlated non-parents, and Appendix C reports controlled latent-variable diagnostics. LargeST results are presented as spatial plausibility rather than causal recovery (Section 5.4).

\item {\bf Theory assumptions and proofs}
    \item[] Question: For each theoretical result, does the paper provide the full set of assumptions and a complete (and correct) proof?
    \item[] Answer: \answerYes{} 
    \item[] Justification: The main finite-error recovery result is stated in Theorem~\ref{thm:finite_error_graph_recovery}, with assumptions and error terms described in Section~\ref{sec:theory}. The full assumptions, supporting lemmas, proof, and asymptotic interpretation are provided in Appendices~\ref{app:theory_assumptions}--\ref{app:theory_interpretation}.

    \item {\bf Experimental result reproducibility}
    \item[] Question: Does the paper fully disclose all the information needed to reproduce the main experimental results of the paper to the extent that it affects the main claims and/or conclusions of the paper (regardless of whether the code and data are provided or not)?
    \item[] Answer: \answerYes{} 
    \item[] Justification: The paper describes the benchmark groups, baselines, metrics, thresholding protocol, runtime measurement protocol, and evaluation settings in Section~\ref{sec:experiments} and Appendix~\ref{app:experiments}. Additional dataset details, splits, hyperparameters, and simulation setups are provided in Appendices~\ref{app:baseline_hparams}, \ref{app:dataset_overview}, and \ref{app:simulation_setups}.

\item {\bf Open access to data and code}
    \item[] Question: Does the paper provide open access to the data and code, with sufficient instructions to faithfully reproduce the main experimental results, as described in supplemental material?
    \item[] Answer: \answerYes{}{} 
    \item[] Justification: Code, configuration files, and data-generation scripts for all synthetic benchmarks are available at \url{https://github.com/ShaogangRen/SCDM}, with commands to reproduce each table. LargeST is publicly available from its original authors~\citep{liu2023largest}.

\item {\bf Experimental setting/details}
    \item[] Question: Does the paper specify all the training and test details (e.g., data splits, hyperparameters, how they were chosen, type of optimizer) necessary to understand the results?
    \item[] Answer: \answerYes{} 
    \item[] Justification: Section~\ref{sec:exp_setup} describes the main experimental setup, including baselines, metrics, seeds, thresholding, runtime measurement, and high-dimensional evaluation protocol. Detailed hyperparameters, dataset splits, dataset roles, and benchmark-specific settings are provided in
Appendix~\ref{app:baseline_hparams}, Table~\ref{tab:hparam_scdm},
Table~\ref{tab:dataset_splits}, Table~\ref{tab:evaluation_usage}, and Appendix~\ref{app:simulation_setups}.

\item {\bf Experiment statistical significance}
    \item[] Question: Does the paper report error bars suitably and correctly defined or other appropriate information about the statistical significance of the experiments?
    \item[] Answer: \answerYes{} 
    \item[] Justification: Uncertainty estimates, where reported, are defined in the corresponding table and figure captions. Table 3 reports mean $\pm$ SD across three seeds; Table 2 and Figure 6 report means across three seeds without intervals.

\item {\bf Experiments compute resources}
    \item[] Question: For each experiment, does the paper provide sufficient information on the computer resources (type of compute workers, memory, time of execution) needed to reproduce the experiments?
    \item[] Answer: \answerYes{} 
    \item[] Justification: Appendix B.1 lists the GPU type (NVIDIA A100 80\,GB), CPU memory, and per-job time limits for each experiment. Section 5.5 reports training time for the task-batch ablation, and Appendix B.3 reports runtime and peak GPU memory for the large-scale comparisons.
    
\item {\bf Code of ethics}
    \item[] Question: Does the research conducted in the paper conform, in every respect, with the NeurIPS Code of Ethics \url{https://neurips.cc/public/EthicsGuidelines}?
    \item[] Answer: \answerYes{} 
    \item[] Justification: The work is a methodological study of temporal causal discovery and does not involve human subjects, private personal data, or deployment of a decision-making system. The experiments use controlled synthetic, neuroimaging-inspired, high-dimensional simulated, and public traffic benchmark settings as described in Section~\ref{sec:experiments} and Appendix~\ref{app:dataset_overview}.

\item {\bf Broader impacts}
    \item[] Question: Does the paper discuss both potential positive societal impacts and negative societal impacts of the work performed?
    \item[] Answer: \answerYes{} 
    \item[] Justification: Appendix~\ref{app:broader_impacts} discusses both potential
positive impacts, such as more efficient scientific analysis of high-dimensional temporal systems, and potential negative impacts, such as misleading conclusions or harmful decisions if observational causal scores are used in high-stakes domains without validation. The paper emphasizes that SCDM should be used for hypothesis generation and structural analysis rather than direct deployment without domain knowledge or external evidence.
    
\item {\bf Safeguards}
    \item[] Question: Does the paper describe safeguards that have been put in place for responsible release of data or models that have a high risk for misuse (e.g., pre-trained language models, image generators, or scraped datasets)?
    \item[] Answer: \answerNA{} 
    \item[] Justification: The paper does not release high-risk models, scraped personal datasets, pretrained generative models, or deployed systems with substantial misuse potential. The contribution is an algorithmic framework evaluated on synthetic, neuroimaging-inspired, and traffic benchmark datasets.

\item {\bf Licenses for existing assets}
    \item[] Question: Are the creators or original owners of assets (e.g., code, data, models), used in the paper, properly credited and are the license and terms of use explicitly mentioned and properly respected?
    \item[] Answer: \answerYes{} 
    \item[] Justification: We cite the original papers and public repositories of all baseline methods (URLs in Appendix B.1) and use them under their stated licenses. LargeST~\citep{liu2023largest} is used under its CC BY-NC 4.0 license. The Schaefer-100 parcellation~\citep{schaefer2018local} and the Yeo-7 network labels~\citep{yeo2011organization} are cited where used. 
    

\item {\bf New assets}
    \item[] Question: Are new assets introduced in the paper well documented and is the documentation provided alongside the assets?
    \item[] Answer: \answerNA{} 
    \item[] Justification: The paper does not release a new public dataset or pretrained model as a primary asset. The synthetic and neuroimaging-inspired benchmark construction details are documented in Appendices~\ref{app:simulation_setups}, \ref{app:simulated_fmri_setup}, \ref{app:simulated_eeg_setup}, and \ref{app:voxel_fmri_setup}.

\item {\bf Crowdsourcing and research with human subjects}
    \item[] Question: For crowdsourcing experiments and research with human subjects, does the paper include the full text of instructions given to participants and screenshots, if applicable, as well as details about compensation (if any)? 
    \item[] Answer: \answerNA{} 
    \item[] Justification: The work does not involve crowdsourcing, user studies, or experiments with human subjects. All experiments are conducted on synthetic, neuroimaging-inspired, high-dimensional simulated, or public traffic benchmark data.

\item {\bf Institutional review board (IRB) approvals or equivalent for research with human subjects}
    \item[] Question: Does the paper describe potential risks incurred by study participants, whether such risks were disclosed to the subjects, and whether Institutional Review Board (IRB) approvals (or an equivalent approval/review based on the requirements of your country or institution) were obtained?
    \item[] Answer: \answerNA{} 
    \item[] Justification: The work does not involve human-subject research, personal data collection, or crowdsourcing. Therefore, IRB approval or an equivalent human-subject review is not applicable.

\item {\bf Declaration of LLM usage}
    \item[] Question: Does the paper describe the usage of LLMs if it is an important, original, or non-standard component of the core methods in this research? Note that if the LLM is used only for writing, editing, or formatting purposes and does \emph{not} impact the core methodology, scientific rigor, or originality of the research, declaration is not required.
    \item[] Answer: \answerNA{} 
    \item[] Justification: The core methodology, theoretical analysis, experiments, and evaluation pipeline do not use LLMs as an important, original, or non-standard component. Any use of general writing or formatting assistance does not affect the scientific method or experimental results.

\end{enumerate}

\end{document}